\documentclass[11pt]{article}

\usepackage[letterpaper,margin=1in]{geometry}
\usepackage[T1]{fontenc}
\usepackage{lmodern}
\usepackage{amsmath,amssymb,amsthm,mathtools}
\usepackage{booktabs}
\usepackage{graphicx}
\usepackage{microtype}
\usepackage{placeins}
\usepackage{float}
\usepackage{xcolor}
\usepackage[numbers,sort&compress]{natbib}
\usepackage{hyperref}
\usepackage{url}
\hypersetup{
  hidelinks,
  pdftitle={Recovering Weighted Tangent Geometry from a Single-Scale Score Field},
  pdfauthor={Ziqi Zhao and Qingjian Ni}
}

\title{Recovering Weighted Tangent Geometry from a Single-Scale Score Field}

\author{
  Ziqi Zhao\\
  School of Computer Science and Engineering\\
  Southeast University\\
  Nanjing, China\\
  \texttt{ziqizhao@seu.edu.cn}
  \and
  Qingjian Ni\thanks{Corresponding author.}\\
  School of Computer Science and Engineering\\
  Southeast University\\
  Nanjing, China\\
  \texttt{nqj@seu.edu.cn}
}
\date{}

\newtheorem{theorem}{Theorem}
\newtheorem{proposition}{Proposition}
\newtheorem{lemma}{Lemma}

\newtheorem{appendixtheorem}{Theorem}[section]
\theoremstyle{definition}

\newtheorem{algorithm}{Algorithm}

\newcommand{\R}{\mathbb{R}}
\newcommand{\Sph}{\mathbb{S}}
\newcommand{\T}{\mathsf{T}}
\newcommand{\E}{\mathbb{E}}
\newcommand{\op}{\mathrm{op}}
\newcommand{\dd}{\mathrm{d}}
\newcommand{\eps}{\varepsilon}
\providecommand{\ArcCriticalRatio}{1.5016756510}

\providecommand{\ArcDerivativeLower}{0.89335}
\providecommand{\FarFieldLogDensity}{-435.40}
\providecommand{\FarFieldScoreNorm}{800.99}
\providecommand{\FiniteQueryTrials}{104,352}
\providecommand{\FiniteConditioningTrials}{103,680}
\providecommand{\FiniteLocalizationTrials}{672}
\providecommand{\CertificateExecutionTrials}{336}
\providecommand{\EmpiricalKDETrials}{6,528}
\providecommand{\CoverageTrials}{4,480}
\providecommand{\PerturbationTrials}{2,048}
\providecommand{\CoverageNonAbstaining}{810}
\providecommand{\CoverageCorrectReported}{791}
\providecommand{\PerturbationNonAbstaining}{413}
\providecommand{\PerturbationCorrectReported}{413}
\providecommand{\ScoreCoverageSlope}{-0.490}
\providecommand{\MomentCoverageSlope}{-0.499}
\providecommand{\HeldoutSuccessLow}{98.96\%}
\providecommand{\HeldoutSuccessMedium}{98.96\%}
\providecommand{\HeldoutSuccessHigh}{99.74\%}
\providecommand{\PerturbationBlindMin}{83\%}
\providecommand{\PerturbationBlindMax}{93\%}
\providecommand{\CenterPerfectMaxOffset}{0.1}
\providecommand{\CenterPerfectSuccess}{100\%}
\providecommand{\CenterFirstDegradedOffset}{0.2}
\providecommand{\CenterFirstDegradedSuccess}{68.75\%}
\providecommand{\LearnedSeedCount}{5}
\providecommand{\LearnedPlainParameters}{33,794}
\providecommand{\LearnedResidualParameters}{33,666}
\providecommand{\LearnedStandardModels}{80}
\providecommand{\LearnedLongModels}{40}
\providecommand{\LearnedEvaluationRows}{520}
\providecommand{\LearnedOptimizerUpdates}{280,000}
\providecommand{\LearnedTotalModels}{260}
\providecommand{\LearnedTotalUpdates}{500,000}
\providecommand{\SensitivityWallMinutes}{8.16}
\providecommand{\LearnedEvaluationWallMinutes}{15.51}
\providecommand{\LearnedTotalWallMinutes}{23.67}
\providecommand{\LearnedPeakCudaReservedMiB}{24.0}
\providecommand{\SensitivityCpuPilotSeconds}{18.5}
\providecommand{\LearnedEvaluationCpuPilotSeconds}{12.0}
\providecommand{\SensitivityCpuPilotUpdates}{12,000}
\providecommand{\LearnedEvaluationCpuPilotUpdates}{3,200}

\providecommand{\LearnedStagewiseCountFloor}{4/5}
\providecommand{\LearnedStagewiseFullCeiling}{2/5}
\providecommand{\ClusterPlainLongCount}{5/5}
\providecommand{\ClusterResidualLongCount}{5/5}
\providecommand{\WeakFourPlainLongCount}{5/5}
\providecommand{\WeakFourResidualLongCount}{4/5}
\providecommand{\WeakYPlainLongCount}{3/5}
\providecommand{\WeakYResidualLongCount}{2/5}
\providecommand{\SensitivityUndercountSettings}{3}
\providecommand{\SensitivityUndercountSeeds}{3/5}
\providecommand{\SensitivityLongControls}{4}
\providecommand{\SensitivityLongFailureMin}{80\%}
\providecommand{\SensitivityLongFailureMax}{100\%}

\providecommand{\ConvergenceModels}{110}

\providecommand{\ConvergenceScoreImprovementMin}{5.4\%}
\providecommand{\ConvergenceScoreImprovementMax}{24.0\%}
\providecommand{\ConvergenceMomentIncreaseMin}{3.6\%}
\providecommand{\ConvergenceMomentIncreaseMax}{20.9\%}

\begin{document}

\maketitle

\begin{abstract}
Near a smooth data manifold, one tangent space summarizes local geometry.  At a
branch point, the corresponding first-order object is instead a measure over
tangent directions, whose normalized masses record the local share of each
branch under the chosen data measure.  We ask whether a score field at one noise
level determines this weighted tangent geometry when the branch center and
homogeneity degree $d$ are unknown.  In this tangent-measure model, $d$ is the
local measure dimension.  Gaussian smoothing of a homogeneous tangent measure
satisfies an Ornstein--Uhlenbeck eigenfunction equation.  Its weak form turns
score values---without score derivatives---into a linear system for the center
and homogeneity degree, with an explicit rank condition and perturbation bound.
After this calibration, the tangential score on one sphere is the spherical
log-gradient of a scalar Gaussian--cone transform.  Integration recovers that
transform up to scale, and all its spherical-harmonic multipliers are positive.
Thus one exact shell identifies the normalized angular measure in every ambient
dimension $D\geq2$.  For at most $K$ positive rays, moments through degree
$2K-1$ constructively recover count, directions, and weights in arbitrary
dimension.  Any fixed observation scheme needs at least $KD-1$ scalar
tangential components.  In the plane, degree $K$ is both sufficient and
necessary, and we give quantitative finite-query certificates.  For finite
planar $C^{1,\beta}$ branches with positive $C^{0,\beta}$ densities, we prove
$O(\sigma^\beta)$ convergence from the finite-noise score to its tangent model.
In controlled experiments, 50k-step training lowers validation normalized-score
error across four geometries yet raises angular-moment error, separating
ordinary score fit from geometry recovery.
\end{abstract}

\section{Introduction}
\label{sec:intro}

The score $s_\sigma(x)=\nabla_x\log(\mu*\varphi_\sigma)(x)$ is a local vector
field used by score-based generative models \citep{song2021sde}.  Near a smooth
manifold, its normal component and Jacobian reveal tangent and normal spaces
\citep{stanczuk2024dimension,ventura2025geometric}.  Many supports are not
smooth everywhere: trajectories merge, road segments intersect, and stratified
spaces contain boundaries and junctions.  No single tangent space describes
such a point.  The first-order object is instead a normalized angular tangent
measure.  In a finite-ray model, its atoms record branch directions and their
relative shares of local mass under the data measure.

We call the local scaling exponent the homogeneity degree; in our tangent model,
it is the local measure dimension.  We study an inverse question at one known
noise level:
\begin{quote}
Can local score queries recover a branch point, its homogeneity degree, and the
weighted tangent geometry represented by the field?
\end{quote}
This differs from detecting a singularity in raw samples
\citep{vonrohrscheidt2023topological,lim2025hades}.  A point-cloud method can
identify a neighborhood worth auditing and estimate sample geometry, but it
does not establish what a trained score field represents there.  Conversely, a
low denoising objective does not certify a center, homogeneity degree, branch
count, directions, or mass.  We therefore start from a local query window, not
from an exactly supplied center; global singularity search remains a separate
problem.  Here one noise level refers to the available field used by this local
inverse problem; training and sampling may still use a full noise schedule.

A one-point score Hessian is insufficient for branch geometry: at a conic
vertex it depends only on the first two angular moments.  Every uniform regular
planar $q$-ray junction with $q\geq3$ therefore has the same vertex Hessian,
although different counts produce different score patterns away from the
vertex.  A separate trace identity does encode homogeneity, and its weak form
lets a spatial score pattern first calibrate the unknown center and homogeneity
degree.
Once centered, one normalized shell contains the remaining angular information.

The recovery chain begins with weak score averages that calibrate the center and
homogeneity degree.  Tangential queries on the calibrated shell are then
integrated into a normalized scalar density, whose harmonic coefficients yield
angular moments.  A finite-rank moment pencil reconstructs the branches.  The
procedure finally reports a numerical candidate, a certified stage, or
abstention.  The calibration and integration steps convert vector log-gradient
observations into the centered scalar-moment interface assumed by spectral
estimators.

For an exact homogeneous tangent model, the center, homogeneity, and normalized
angular measure are recoverable in any $D\geq2$ under explicit identifiability
conditions.  Quantitative discrete-query certificates are specialized to the
plane, but exact finite-branch reconstruction is not.  Our contributions are:
\begin{enumerate}
\item \textbf{Score-only center and homogeneity calibration.}  Homogeneity makes
the smoothed conic density an Ornstein--Uhlenbeck eigenfunction.  Testing its
score equation against localized functions gives a linear system for the
unknown center and degree $d$, without differentiating the score.  Full column
rank is sufficient, its failure captures translational symmetries, and a direct
least-squares perturbation bound quantifies approximate fields.
\item \textbf{Constructive recovery in arbitrary dimension.}  In every
$D\geq2$, a centered exact shell identifies any finite positive angular
measure.  If it has at most $K$ atoms, harmonic moments through degree $2K-1$
form a finite-rank multivariate moment pencil whose rank, joint eigenvalues, and
linear coefficients recover count, directions, and weights.  A fixed scheme
requires at least $KD-1$ scalar tangential observations.
\item \textbf{Sharp planar recovery and conditional stability.}  For finite
$C^{1,\beta}$ planar branches with positive
$C^{0,\beta}$ densities, the tangent-score bias is $O(\sigma^\beta)$ on fixed
normalized query sets.  For positive planar rays, moments through degree $K$
suffice, whereas degree $K-1$ does not.  Explicit finite-sampling aliasing and
declared minimum weight, angular separation, and conditioning bounds yield
sufficient count, direction, and weight tests.
\item \textbf{Finite-data and learned-score diagnostics.}  The error chain
exposes tangent bias, local coverage, calibration error, harmonic amplification,
network error, and query aliasing.  Population, empirical kernel-density
estimate (KDE), and learned fields pass through matched recovery stages.
Long-horizon controls show that lower validation normalized-score error can
coexist with higher angular-moment error and incomplete geometry recovery.
\end{enumerate}
Classical spectral methods start from a scalar convolution or its moments.  Our
observation is instead an uncentered vector-valued log-gradient.  The
score-specific part of the pipeline removes the unknown center and homogeneity,
recovers the scalar shell transform up to its irrelevant multiplicative
constant, proves that the Gaussian--cone kernel has no harmonic nullspace, and
propagates score-side errors to moments.  The subsequent multivariate Prony,
matrix-pencil, Toeplitz, and Vandermonde steps are classical
\citep{kunis2016multivariate,kunis2019prony}.
Our results concern local identifiability from an available score slice;
global localization, multiscale necessity, and universality across diffusion
architectures are outside scope.  Appendix~\ref{app:counterexamples} and
Section~\ref{sec:limitations} discuss related boundary cases and limitations.

\section{Self-calibration and one-shell tangent geometry}
\label{sec:shell}

Let an unknown branch point $x_0\in\operatorname{supp}\mu\subset\R^D$ lie in
the local query window.  Tangent measures formalize blow-up limits of measures
\citep{preiss1987geometry}.  We assume the rescaled local measures converge, in
the Gaussian-weighted moments used below, to a nonzero $d$-homogeneous tangent
measure
\begin{equation}
  \dd\nu_\Lambda(r,\theta)=r^{d-1}\dd r\,\dd\Lambda(\theta),
  \qquad \theta\in\Sph^{D-1}.
  \label{eq:cone}
\end{equation}
The noise scale $\sigma$ is known from the score model.  The center $x_0$ and
homogeneity degree $d$ will be recovered rather than supplied.
The score cannot recover the total mass of the finite positive angular measure
$\Lambda$.  Its identifiable target is the normalized measure
\begin{equation*}
 \bar\Lambda(A):=\frac{\Lambda(A)}{\Lambda(\Sph^{D-1})}
 =\frac{\nu_\Lambda(\{r\theta:0<r\leq1,\ \theta\in A\})}
 {\nu_\Lambda(B(0,1))},
 \qquad A\subseteq\Sph^{D-1}.
\end{equation*}
Thus a branch weight is a sector-wise fraction of tangent mass under the
reference measure defining $\mu$.  Overall rescaling and reparameterizations
that preserve $\mu$ leave it unchanged; replacing parameter measure by
Hausdorff or arc-length measure can change it.  For a data-distribution measure,
the weights describe local probability shares, while for a geometric reference
measure they describe relative branch density or multiplicity.  We henceforth
normalize $\Lambda$ to unit mass when discussing weights.  In the planar
finite-ray case, $\Lambda=\sum_{j=1}^{s}w_j\delta_{\theta_j}$ with $w_j>0$ and
$\sum_jw_j=1$ records branch count, directions, and relative mass.

The tangent-measure change of variables gives
\begin{equation}
  \sigma s_\sigma(x_0+\sigma z)\longrightarrow
  F_\Lambda(z):=\nabla_z\log q_\Lambda(z),\qquad
  q_\Lambda(z)=\int e^{-\|z-u\|^2/2}\dd\nu_\Lambda(u).
  \label{eq:tangentlimit}
\end{equation}

For finite planar $C^{1,\beta}$ branches with $C^{0,\beta}$ positive densities,
\eqref{eq:tangentlimit} holds uniformly at rate $O(\sigma^\beta)$ on compact
normalized query sets.  Appendix Theorem~\ref{thm:planar-tangent-bias} gives the
conditions and proof; Appendix Figure~\ref{fig:tangent-bias-rate} checks the
predicted slopes but is not used in the proof.

\subsection{Center and homogeneity from score values}

Choose any reference $c$ in the local query window and write
$y=(x-c)/\sigma$, $b=(x_0-c)/\sigma$, and
$u(y)=\sigma s_\sigma(c+\sigma y)$.  In the exact conic model,
$u(y)=F_\Lambda(y-b)$.  Homogeneity and the Gaussian heat equation imply
\begin{equation}
 \nabla\!\cdot u(y)+\|u(y)\|^2+(y-b)^\top u(y)+D-d=0.
 \label{eq:ou-score}
\end{equation}
\paragraph{Idea.}
Homogeneity balances score energy, radial score, and local dimension.  A wrong
center enters this balance linearly; localized averages therefore produce a
small linear system without differentiating the score.  For a differentiable
localized $\psi$, define
\begin{align}
 a_\psi&=\left(\int\psi(y)u(y)\dd y,\ \int\psi(y)\dd y\right),\\
 r_\psi&=\int\left[-\nabla\psi(y)^\top u(y)
 +\psi(y)\{\|u(y)\|^2+y^\top u(y)+D\}\right]\dd y.
 \label{eq:weak-calibration}
\end{align}

\begin{theorem}[Single-scale center and homogeneity calibration]
\label{thm:calibration}
Stack the rows $a_{\psi_j}$ into $A$ and the values $r_{\psi_j}$ into $r$.
For every exact translated $d$-homogeneous cone,
\begin{equation}
 A\binom{b}{d}=r.
 \label{eq:calibration-system}
\end{equation}
If $A$ has full column rank, score values at this one scale uniquely recover
$b$ and $d$.  For perturbed quantities $\widehat A=A+E$ and
$\widehat r=r+e$, least squares obeys
\begin{equation}
 \left\|\binom{\widehat b}{\widehat d}-\binom{b}{d}\right\|
 \leq\frac{\|e\|+\|E\|\left\|\binom{b}{d}\right\|}
 {\sigma_{\min}(A)-\|E\|},\qquad \|E\|<\sigma_{\min}(A).
 \label{eq:calibration-bound}
\end{equation}
\end{theorem}
Here $E,e$ collect score, integration, and tangent-model errors; for a stochastic
estimator, the bound holds on any event controlling their norms.  Gaussian tests
make the integrals expectations of score values.  If the cone is translation invariant in a direction
$v$, then $u(y)^\top v=0$, the column for that center direction is
unidentifiable, and shifting the nominal center along $v$ changes no geometry.
Thus the rank condition excludes a genuine symmetry rather than merely a
numerical inconvenience.  Appendix~\ref{app:calibration} proves the strong and
weak identities, the perturbation bound, and the corresponding approximate
tangent statement.  We henceforth recenter at the recovered $x_0$ and use the
recovered $d$.

\subsection{One-shell angular recovery}

Related boundary-layer analysis derives forward asymptotics for tangent cones
and corners \citep{brosse2026boundary}.  Our inverse problem begins with the
limiting field and asks whether it determines the angular measure.

Fix $R>0$, write $z=R\omega$, and define
\begin{equation}
 K_{d,R}(t)=\int_0^\infty r^{d-1}e^{-r^2/2+Rrt}\dd r,
 \qquad
 \T_{d,R}\Lambda(\omega)=\int K_{d,R}(\omega^\top\theta)\dd\Lambda(\theta).
 \label{eq:transform}
\end{equation}
Polar integration and tangential differentiation yield the central identity
\begin{equation}
 q_\Lambda(R\omega)=e^{-R^2/2}\T_{d,R}\Lambda(\omega),\qquad
 \boxed{g_R(\omega):=R P_{\omega^\perp}F_\Lambda(R\omega)
 =\nabla_{\Sph^{D-1}}\log\T_{d,R}\Lambda(\omega).}
 \label{eq:shellidentity}
\end{equation}
We call $g_R$ the \emph{shell tangential score}.  It reveals the positive scalar
shell transform up to one multiplicative constant.  The physical query radius
is $R\sigma$.

\paragraph{Idea.}
The tangential score gives directional changes of a blurred angular density.
Integrating these changes on the connected shell recovers that density up to
scale, normalization removes the scale, and Gaussian--cone smoothing preserves
every harmonic mode.

\begin{theorem}[One-shell injectivity]
\label{thm:injective}
Fix $D\geq2$, $d>0$, and $R>0$.  If two nonzero finite positive angular
measures have identical shell tangential scores $g_R$ on $R\Sph^{D-1}$, then
they are proportional.  Unit-mass angular measures are therefore equal.
\end{theorem}

The proof has two steps.  Equal tangential log-gradients imply proportional
scalar transforms because the sphere is connected.  The Funk--Hecke formula,
the spherical analogue of Fourier diagonalization for a rotationally symmetric
kernel, diagonalizes the transform.  Its degree-$\ell$ multiplier is a positive radial mixture
of $(Rr)^{1-D/2}I_{\ell+D/2-1}(Rr)$; hence no harmonic mode vanishes.  Full
details are in Appendix~\ref{app:injectivity}.

\paragraph{Idea.}
The shell has already converted score queries into polynomial moments.
Coordinate multiplication then becomes a family of commuting small matrices;
their joint eigenvalues are the branch directions, and a linear solve gives the
weights.
\begin{theorem}[Finite positive rays in arbitrary dimension]
\label{thm:finite-spherical}
Let $D\geq2$ and let
$\Lambda=\sum_{j=1}^{s}w_j\delta_{\theta_j}$ contain $s\leq K$ distinct
directions with positive weights.  From one exact shell, spherical-harmonic
moments through degree $2K-1$ constructively determine $s$, every
$\theta_j\in\Sph^{D-1}$, and every normalized weight $w_j$.
\end{theorem}
Once the shell moments are available, the remaining atom-recovery step is
classical spherical multivariate Prony
\citep{kunis2016multivariate,kunis2019prony}: moment-matrix rank gives the count,
commuting multiplication matrices give the directions, and a linear solve gives
the weights.  Appendix~\ref{app:finite-spherical} verifies the interpolation
rank and degree-$2K-1$ bound.  The paper-specific step is the preceding
conversion from an uncentered vector score to moments.

\paragraph{Idea.}
On the circle, spherical harmonics become Fourier modes.  Positivity turns these
moments into a positive semidefinite Toeplitz matrix whose rank gives the count
and whose nullspace polynomials vanish at the branch directions.
\begin{theorem}[Finite positive planar rays]
\label{thm:planar}
Let $D=2,d=1$, and suppose $\Lambda$ contains $s\leq K$ distinct positive
rays.  The reconstructed unit-mean shell density $h$ has Fourier modes
\begin{equation}
 h_k=\rho_{|k|}(R)m_k,\qquad
 m_k=\sum_{j=1}^{s}w_je^{-ik\theta_j},\qquad
 \rho_k(R)=\frac{\kappa_k(1,R)}{\kappa_0(1,R)}>0.
 \label{eq:moments}
\end{equation}
Modes $|k|\leq K$ determine $s$, all directions, and all weights.
\end{theorem}
\begin{proposition}[Why degree $K$ is necessary in the plane]
\label{prop:sharp-moments}
For every $K\geq1$, normalized positive measures supported on at most $K$
planar rays are not identified by the consecutive moments
$m_0,\ldots,m_{K-1}$.  Hence the largest order $K$ in
Theorem~\ref{thm:planar} is necessary for this moment description.
\end{proposition}
After deconvolution, the Toeplitz/Vandermonde step is standard in
finite-rate-of-innovation, subspace, and super-resolution methods
\citep{vetterli2002fri,schmidt1986music,roy1989esprit,
candes2014superresolution,moitra2015superresolution,yang2016vandermonde}:
Toeplitz rank gives the count, common nullspace roots give directions, and a
linear solve gives weights.  Our contribution is to supply these moments from a
local score shell and carry score error to stability margins.
Proposition~\ref{prop:sharp-moments} follows from rotated regular $K$-ray
measures, whose first $K-1$ nonconstant moments all vanish; the proof is in
Appendix~\ref{app:planar}.  This is a lower bound for the consecutive-moment
interface, because one score query is a nonlinear combination of all angular
modes.  The next result gives a separate lower bound directly at the
score-query interface.

\begin{proposition}[Fixed scalar-query lower bound]
\label{prop:query-lower}
Fix $K$ pairwise disjoint open spherical patches, one for each ray, and allow
all positive normalized weights.  If $M$ scalar tangential-score components
are observed through a fixed continuous scheme, exact identification of every
measure in this class requires
\begin{equation}
 M\geq KD-1.
\end{equation}
This holds even when $K$ is known and the observations are noiseless.
\end{proposition}
The directions contribute $K(D-1)$ local coordinates and the normalized
weights contribute $K-1$.  Their map to the observations is continuous, so
invariance of domain gives the bound
\citep[Section~62]{munkres2000topology}.  A full tangential vector at one fixed
shell location contains at most $D-1$ scalar components.  The proposition does
not cover adaptive observations.  For $D=2$ it reduces to $M\geq2K-1$; the
smallest uniform planar grid covered by our conditional analysis uses
$M=2K+1$.  Appendix~\ref{app:finite-spherical} gives the general proof.

\section{Shell information beyond a one-point Hessian}
\label{sec:hessian}

For unit-mass $\Lambda$, let
$m_\Lambda=\E_\Lambda\theta$, $Q_\Lambda=\E_\Lambda\theta\theta^\top$, and
$a_d=\sqrt{2}\,\Gamma((d+1)/2)/\Gamma(d/2)$.  Direct differentiation gives
\begin{equation}
 \nabla F_\Lambda(0)=dQ_\Lambda-a_d^2m_\Lambda m_\Lambda^\top-I_D.
 \label{eq:hessian}
\end{equation}
Taking the trace and using $F_\Lambda(0)=a_dm_\Lambda$ recovers the scalar
homogeneity through
$d=D+\operatorname{tr}\nabla F_\Lambda(0)+\|F_\Lambda(0)\|^2$, the centered
pointwise case of \eqref{eq:ou-score}.  The weak calibration theorem obtains the
same nuisance parameter without knowing the center or differentiating the
score.  By contrast, the full vertex Hessian still keeps only two angular
moments.  Every uniform regular
planar $q$-ray junction with $q\geq3$ has $m_\Lambda=0$,
$Q_\Lambda=I_2/2$, and, for $d=1$, $\nabla F(0)=-I_2/2$.  Yet its first
count-specific shell harmonic is nonzero.
Figure~\ref{fig:collision} compares the common vertex Hessian, the distinct
shell fields, and the inverse harmonic multipliers at the tested radii.
Appendix~\ref{app:hessian} proves \eqref{eq:hessian} and the full regular-$q$
statement.

\begin{figure}[H]
  \centering
  \includegraphics[width=0.99\linewidth]{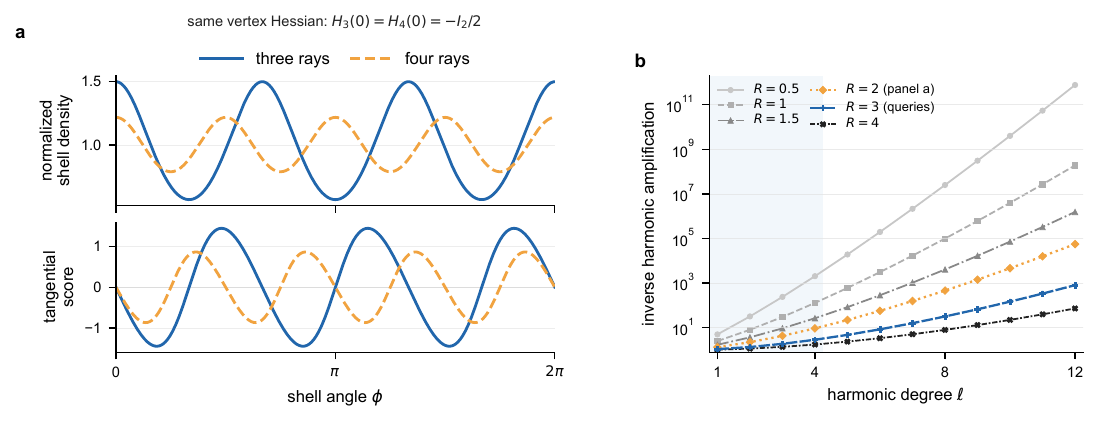}
  \caption{One-point collision and shell conditioning.  \textbf{(a)} For
  $D=2,d=1$, regular three- and four-ray measures share the vertex score
  Hessian $-I_2/2$ but have different $R=2$ shell densities and tangential
  scores.  \textbf{(b)} Inverse harmonic multipliers grow rapidly with degree
  at small radii.  Marker and line style identify radii; shading marks degrees
  $1{:}4$ used by the planar $K=4$ experiments.}
  \label{fig:collision}
\end{figure}

\section{From finite queries to conditional guarantees}
\label{sec:finite}

Exact identifiability assumes a continuous field.  Here finitely many noisy
queries produce a candidate, and predeclared margins determine what can be
certified.

\subsection{Finite-query reconstruction}

For the planar estimator, query $M>2K$ equally spaced angles
$\phi_j=2\pi j/M$.  Let $g(\phi)$ be the planar parameterization of the shell
tangential score $g_R$, so $g=(\log h)'$, and observe
$y_j=g(\phi_j)+e_j$ with
$M^{-1}\sum_j|e_j|^2\leq\eps_2^2$; thus $\eps_2$ is a root-mean-square (RMS)
observation bound.  The estimator spectrally integrates $g$, exponentiates and
normalizes $h$, and removes the known shell multipliers from modes $1{:}K$.
Toeplitz rank gives branch count, annihilating-polynomial roots give directions,
and a final nonnegative solve gives weights.

\begin{algorithm}[Finite-query one-shell estimator]
\label{alg:finitequery}
Inputs are $R$, $K$, and the $M$ shell tangential score values.
\begin{enumerate}
\item Divide every nonzero discrete Fourier transform (DFT) mode by its
frequency, then exponentiate and normalize to reconstruct the shell density.
\item Divide density modes $1{:}K$ by the known shell multipliers to obtain
angular moments.
\item Infer Toeplitz rank, extract unit-circle roots, and fit normalized
nonnegative weights.
\end{enumerate}
A numerical candidate uses a relative eigengap.  A certified output applies the
thresholds in Theorem~\ref{thm:finitequery} and stops whenever a sufficient
test fails.
\end{algorithm}

Let $A_M(g)$ denote the explicit spectral-integration alias term,
$B_{M,k}(R)$ the scalar-density alias term, and
\begin{equation}
 C_M=\left(\sum_{k\in\mathcal K_M\setminus\{0\}}k^{-2}\right)^{1/2}
 <\pi/\sqrt3,\qquad \tau_M=A_M(g)+C_M\eps_2.
\end{equation}
Appendix~\ref{app:finitequery} proves the finite-query moment bound
\begin{equation}
 \boxed{
 |\widehat m_k-m_k|\leq
 \rho_k(R)^{-1}(e^{2\tau_M}-1)
 +\frac{B_{M,k}(R)/\rho_k(R)+B_{M,0}(R)}{1-B_{M,0}(R)}},
 \quad 1\leq k\leq K,
 \label{eq:momentbound}
\end{equation}
where $\rho_k=\kappa_k/\kappa_0$ and $B_{M,0}<1$.  This separates score
noise, log integration, density sampling alias, and harmonic amplification.

\subsection{Conditional deterministic certification}

\paragraph{Idea.}
The score-error bound first becomes a uniform moment-error bound.  Eigenvalues
then certify branch count, separated polynomial roots certify directions, and a
conditioned linear solve certifies weights.  Each stage uses bounds fixed before
inspecting the field and abstains before any unverified later output.

\begin{theorem}[Finite-query certification]
\label{thm:finitequery}
Assume $M>2K$, $A_M(g)<\infty$, $B_{M,0}(R)<1$, and the stated RMS observation
error.  Let $\bar\eta$ be the maximum of the right-hand side of
\eqref{eq:momentbound} over $1\leq k\leq K$, or any other valid bound on
$\max_{k\leq K}|\widehat m_k-m_k|$.  Then
\begin{equation}
 \|\widehat T_K-T_K\|_\op\leq(K+1)\bar\eta=:E_K.
 \label{eq:toeplitz}
\end{equation}
For a declared geometry class with minimum positive Toeplitz eigenvalue
$\Gamma_K$, branch count is guaranteed if
$\Gamma_K>2E_K$: it is the number of eigenvalues of $\widehat T_K$ above
$E_K$.  This certified rule is distinct from the relative-eigengap heuristic
used only for an uncertified candidate.  An error estimate alone cannot exclude an extra
branch of arbitrarily small weight, so this class-level resolution gap is
needed for a uniform guarantee.  Declared lower bounds on angular separation
and the smallest Vandermonde singular value give the sufficient direction and
weight tests in Appendix~\ref{app:stability}.  A failed sufficient test causes
abstention; it does not invalidate the numerical candidate.  Here ``declared''
means fixed before inspecting the tested field, using domain knowledge about
the target class.  Substituting the unknown true instance into these bounds is
an oracle diagnostic, not a deployable certificate.
\end{theorem}

Appendix Table~\ref{tab:certificate-stages} exposes the three sufficient tests.
The exact Rouch\'e and Vandermonde inequalities, including the $s=1$ case, are
stated and proved in Appendix~\ref{app:stability}.

The executable version derives every required lower bound from a minimum weight
and angular separation fixed before inspecting the field.  Appendix
\ref{app:stability} gives these conservative substitutions; failure makes the
procedure abstain rather than use the unknown true geometry.

\subsection{Finite-data and learned-score errors}

For $N$ independent data samples, the normalized empirical KDE score at
$x=x_0+\sigma z$ is exactly
\begin{equation}
 \sigma\widehat s_{N,\sigma}(x)=
 \frac{N^{-1}\sum_i (X_i-x)\sigma^{-1}
 e^{-\|x-X_i\|^2/(2\sigma^2)}}
 {N^{-1}\sum_i e^{-\|x-X_i\|^2/(2\sigma^2)}}.
 \label{eq:kderatio}
\end{equation}
For a general local measure, assume a Gaussian-weighted tangent approximation
at rate $\sigma^\alpha$ on the translated query neighborhoods required by the
center-error bound.  Theorem~\ref{thm:planar-tangent-bias} verifies this
condition with $\alpha=\beta$ for its planar branch class.  Also assume
$a_j\geq a_-\sigma^d$, validated calibration bounds
$\|\widehat b-b\|\leq\eps_{\rm ctr}$ and
$|\widehat d-d|\leq\eps_d$, and a validated pointwise network error
$\eps_\theta(\sigma)$.  A simultaneous
Bernstein ratio bound over the $M$ queries then gives, with probability at least
$1-\delta$,
\begin{equation}
 \boxed{
 \eps_{\rm shell}\leq R\Bigg\{C_{\rm geom}\sigma^\alpha
 +C_{\rm samp}\left[
 \sqrt{\frac{\log(M/\delta)}{Na_-\sigma^d}}
 +\frac{\log(M/\delta)}{Na_-\sigma^d}\right]
 +\sigma\eps_\theta(\sigma)+C_{\rm ctr}\eps_{\rm ctr}\Bigg\}.}
 \label{eq:endtoend}
\end{equation}
Here $C_{\rm geom}$ records the tangent approximation and density lower bound,
while $C_{\rm ctr}$ controls local score variation.  The exact sampling term is \eqref{eq:explicitkde} with
$t=\log(4M/\delta)$, $\lambda=Na_-\sigma^d$, and
$S=\max_j\|b_j/a_j\|$; the simplified display requires
$\sqrt{2t/\lambda}+t/(3\lambda)\leq1/2$.  Thus $C_{\rm samp}$ hides only
universal constants and the declared bound on $S$, not an additional power of
$N$ or $\sigma$; none of these geometry-dependent quantities is distribution-free.
On any compact declared interval for $d$, the finitely many positive shell
multipliers used by the inverse have bounded logarithmic derivatives.
Consequently, using $\widehat d$ adds at most $C_d\eps_d$ to the recovered
moment error in \eqref{eq:momentbound}.  Combining this term with
$\eps_2\leq\eps_{\rm shell}$ yields branch count, angle, and weight guarantees,
or abstention.  Balancing only tangent
bias and sampling error gives
$\sigma_{\rm opt}\asymp(Na_-)^{-1/(2\alpha+d)}$; this is estimator
calibration, not a multiscale-necessity claim.

\subsection{Stagewise stability}

\begin{proposition}[Stagewise perturbation margins]
\label{prop:stagewise}
Let $\eta$ bound the recovered moments through order $K$, and write
$\gamma_K(\Lambda)=\lambda_{\min}^{+}(T_K)$.  Correct branch count is
guaranteed whenever
\begin{equation}
 \eta<\eta_{\rm count}:=\frac{\gamma_K(\Lambda)}{2(K+1)}.
 \label{eq:countmargin}
\end{equation}
Conditional on that count, an angular tolerance $a_0$ is guaranteed if
the Rouch\'e inequality in \eqref{eq:rouche} with
$r\leq\sin(a_0/2)$ holds, and a weight tolerance follows from the
Vandermonde inequality in \eqref{eq:weightbound}.  Full geometric recovery
therefore has sufficient rank, root, and weight margins, whereas count uses
only the rank margin.  Replacing the geometry-dependent quantities by declared
class-level lower bounds turns these perturbation statements into uniform
certificates.
\end{proposition}

The proposition follows by composing Weyl rank stability, Davis--Kahan and
Rouch\'e root stability, and the perturbed nonnegative weight solve.  It gives
a concrete mechanism for a count-correct but geometry-inaccurate regime: root
separation or Vandermonde conditioning can become limiting after the Toeplitz
rank has stabilized.  Appendix~\ref{app:stability} gives the complete proof
and constants.

\section{Experiments}
\label{sec:experiments}

We evaluate each link of the recovery chain in reproducible synthetic
experiments and recompute every reported statistic from released seed-level
rows.  Before evaluation, we define planar full recovery as correct count,
maximum circular angle error at most $5^\circ$, and maximum weight error at most
$0.05$.  This descriptive criterion is separate from the certified tests, and Appendix
Tables~\ref{tab:threshold-sensitivity} and
\ref{tab:inference-roles} report tolerance sensitivity and distinguish blind
candidates, oracle diagnostics, and certified outputs.

\subsection{Calibration and recovery beyond the plane}

Table~\ref{tab:joint-recovery} evaluates score-only calibration and
arbitrary-dimensional finite-ray recovery.  Weak Gaussian tests recover center and
homogeneity for exact ray cones in $D=2,3,5$; a separate sweep adds smooth score
perturbations with declared root-mean-square magnitude.  The shell experiment
uses non-coplanar two- and three-ray measures on $\Sph^2$, integrates their
tangential score along meridians, deconvolves spherical harmonics through
degree five, and applies the moment pencil of
Theorem~\ref{thm:finite-spherical}.  Finally, the learned rows train a
three-dimensional residual multilayer perceptron (MLP) by denoising score matching
and apply the same derivative-free calibration.  Models are evaluated in
memory: no checkpoint or prediction dump is part of the experiment.

\begin{table}[t]
  \centering
  \small
  \caption{Single-scale calibration and higher-dimensional reconstruction.
  Center error is normalized by $\sigma$.  Exact-calibration rows report the
  maximum numerical quadrature error; perturbed and learned rows report the
  median.  ``Full geometry'' counts trials meeting the generated count,
  direction, and weight criteria.  For the three-dimensional score-shell rows,
  these criteria are correct count, maximum angular error $0.1^\circ$, and
  maximum weight error $10^{-3}$.  Every entry is generated from seed-level
  rows.}
  \label{tab:joint-recovery}
  \resizebox{\linewidth}{!}{\begin{tabular}{@{}llrrrr@{}}
\toprule
Task and field & $D$ & Runs & Center error & $|\widehat d-d|$ & Full geometry \\
\midrule
Exact weak calibration & 2 & 20 & $4.4\times10^{-4}$ & $8.7\times10^{-5}$ & -- \\
Exact weak calibration & 3 & 20 & $8.5\times10^{-4}$ & $2.2\times10^{-4}$ & -- \\
Exact weak calibration & 5 & 20 & $4.3\times10^{-4}$ & $9.7\times10^{-5}$ & -- \\
Score perturbation ($\mathrm{RMS}=0.01$) & 2 & 20 & $0.018$ & $0.004$ & -- \\
Score perturbation ($\mathrm{RMS}=0.01$) & 3 & 20 & $0.017$ & $0.006$ & -- \\
Score perturbation ($\mathrm{RMS}=0.01$) & 5 & 20 & $0.016$ & $0.005$ & -- \\
Exact score-shell reconstruction & 3 & 20 & -- & -- & 20/20 \\
Learned-score weak calibration & 3 & 5 & $0.061$ & $0.032$ & -- \\
\bottomrule
\end{tabular}
}
\end{table}

The exact and perturbed rows distinguish algebraic identifiability from
conditioning of the chosen test family.  The learned result shows that the
calibration observable survives ordinary score training in this controlled
three-dimensional setting.  A neural-field guarantee additionally requires a
validated external error bound.  The non-coplanar shell rows exercise the full
score-to-harmonic-to-pencil chain from score queries.

\subsection{Finite-query and empirical KDE validation}

Appendix Table~\ref{tab:certificate-execution} and
Figure~\ref{fig:kde} evaluate finite-query reconstruction, certification, and
sampling.  The finite-query matrix contains \FiniteQueryTrials{}
trials.  It varies
normalized radius, query count, score noise, seed, angular separation, minimum
weight, ray clustering, rank selection, and center offset.  At $R=3$ and
$M=32$, all exact evaluation cases recover.  Under noise, a four-ray cluster is
less stable than an isolated close pair with the same minimum separation.  This
is consistent with the full Vandermonde matrix, rather than only the closest
pair, controlling recovery.  Artificial center offsets through
$\CenterPerfectMaxOffset\sigma$ retain
\CenterPerfectSuccess{} success in the baseline matrix; at
$\CenterFirstDegradedOffset\sigma$, success falls to
\CenterFirstDegradedSuccess{} as the declared angle
tolerance is crossed.

The complete sufficient certificate is also executed under geometry bounds
fixed before perturbation.  Across \CertificateExecutionTrials{} bounded-error
trials, every positive report respects its count, direction, or weight bound;
Appendix Table~\ref{tab:certificate-execution} shows how later stages abstain
before earlier ones as error grows.

The population/KDE evaluation contains \EmpiricalKDETrials{} trials over eight
named and twelve continuously generated held-out geometries.  Here local
coverage is the expected effective number of nearby samples receiving
appreciable Gaussian weight at a query.  Population inversion succeeds on all
twenty geometries.  For local coverages at least 32,
fitted log--log slopes are $\ScoreCoverageSlope$ for normalized-score RMS and
$\MomentCoverageSlope$
for moment error, matching inverse-square-root coverage.  On held-out
geometries, blind recovery is \HeldoutSuccessLow{},
\HeldoutSuccessMedium{}, and \HeldoutSuccessHigh{} at coverages
256, 512, and 1024.  A fixed eight-template library and a one-point
Hessian summary based on principal-component analysis (PCA) do not meet the
same full-geometry criterion.  These information-limited baselines are
sanity checks for label templates and one-point information; the reconstruction
procedure itself uses a classical spectral estimator.
Under curvature, $0.1\sigma$ center error, ambient noise, and their
combination, blind recovery is
\PerturbationBlindMin{}--\PerturbationBlindMax{}.  An oracle rank gate
reports \PerturbationNonAbstaining{} candidates in this matrix, all of which
meet the full-geometry tolerance; this diagnostic uses the true synthetic
  signal gap and is not a deployable certificate.  Appendix
Figure~\ref{fig:kde} summarizes the coverage law and held-out recovery.

\paragraph{Controlled learned-score diagnostics.}
\label{sec:learned-study}
Appendix Table~\ref{tab:learned-convergence} and
Figure~\ref{fig:learned-convergence} test whether longer training improves
downstream inversion.  Across ten seeds per cell, the \ConvergenceModels{}
trajectories comprise 80
parameter-matched plain/residual models on four geometries and 30
larger-residual controls on three.  All meet the plateau rule: less than 1\%
improvement in the best validation normalized-score RMS across a ten-evaluation
window, recorded only after 20k updates.  Validation selects the plain MLP at
50k updates for every geometry.  Relative to 5k updates,
normalized-score RMS improves by \ConvergenceScoreImprovementMin{}--
\ConvergenceScoreImprovementMax{} while moment error increases by
\ConvergenceMomentIncreaseMin{}--\ConvergenceMomentIncreaseMax{}.
Full-geometry recovery is 2--4/10, and weak-Y count recovery is 6/10.

\section{Related work}
\label{sec:related}

\textbf{Score geometry.}
Smooth-support work uses scores and Jacobians to estimate normal bundles or
dimensions \citep{stanczuk2024dimension,ventura2025geometric}, build pullback
geometry \citep{diepeveen2025pullback}, or recover projections
\citep{kharitenko2026landing}.  Other analyses treat tubes, singularities,
curvature, support recovery, structured supports, and forward cones or
junctions \citep{sakamoto2024tubular,liu2025singularity,li2026geometry,
zhang2026manifold,yang2026multisubspace,mu2026arbitrary,brosse2026boundary,
chen2026guiding}, but do not invert weighted junction geometry.  We recover its
center, homogeneity, and tangent measure at one scale without a score Jacobian.

\textbf{Spectral and spherical inversion.}
Finite-rate-of-innovation, MUSIC, ESPRIT, and super-resolution recover atoms
from moments \citep{vetterli2002fri,schmidt1986music,roy1989esprit,
candes2014superresolution,moitra2015superresolution,yang2016vandermonde};
spherical deconvolution, kernels, and multivariate Prony invert transforms or
moments \citep{kerkyacharian2011spherical,simongabriel2018kernel,
kunis2016multivariate,kunis2019prony}, while directional score matching assumes
a known manifold \citep{mardia2016directional}.  These methods supply atom
recovery; our contribution is score calibration, shell integration,
all-frequency injectivity, and the finite-query score-to-moment error chain.

\textbf{Neighboring targets.}
Global schedule mixture weights \citep{dennehy2026weights} differ from our local
sector masses.  Point-cloud methods locate but do not audit score fields
\citep{vonrohrscheidt2023topological,lim2025hades}; our coverage is KDE sampling,
not the memorization mechanism in \citet{merger2026coverage}.

\section{Limitations and conclusion}
\label{sec:limitations}

Scope is local: one known noise level, one approximately homogeneous junction
window, a branch bound $K$, and convergence in the Gaussian-weighted score
ratio.  Theorem~\ref{thm:planar-tangent-bias} covers zero-thickness planar
$C^{1,\beta}$ branches, not thickness, general strata, or arbitrary-dimensional
rates; without a stable tangent measure, the target may vary with scale.
Calibration requires full rank and can be ill-conditioned near symmetry;
certification requires external error, weight, and separation bounds and
otherwise abstains.  Global search is outside scope.

From one score slice, we recover local weighted tangent geometry with exact
arbitrary-dimensional identifiability and planar certificates.  Experiments
expose gaps between score fit and geometry recovery; broader rates and
pretrained scores remain future work.

\bibliography{references}
\bibliographystyle{plainnat}

\clearpage
\appendix
\section{Single-scale center and homogeneity calibration}
\label{app:calibration}

This section proves Theorem~\ref{thm:calibration}.  Let $\nu$ be a nonzero
$d$-homogeneous measure on $\R^D$ and define its normalized Gaussian smoothing
at variance $t$ by
\begin{equation}
 q_t(z)=(2\pi t)^{-D/2}\int
 \exp\!\left(-\frac{\|z-u\|^2}{2t}\right)\dd\nu(u).
 \label{eq:scaled-smoothing}
\end{equation}
Changing variables $u=\sqrt t\,v$ and using
$\nu(\sqrt t A)=t^{d/2}\nu(A)$ gives
\begin{equation}
 q_t(z)=t^{(d-D)/2}q_1(z/\sqrt t).
 \label{eq:homogeneous-heat-scaling}
\end{equation}
The Gaussian convolution solves $\partial_tq_t=\Delta q_t/2$.  Differentiating
\eqref{eq:homogeneous-heat-scaling} at $t=1$ and equating the two expressions
for $\partial_tq_t$ yields
\begin{equation}
 \Delta q_1(z)+z^\top\nabla q_1(z)+(D-d)q_1(z)=0.
 \label{eq:ou-density}
\end{equation}
For a translated center $b$, put $q_b(y)=q_1(y-b)$ and
$u=\nabla\log q_b$.  Dividing \eqref{eq:ou-density} by $q_b$ proves the strong
score identity
\begin{equation}
 \nabla\!\cdot u+\|u\|^2+(y-b)^\top u+D-d=0,
\end{equation}
which is \eqref{eq:ou-score}.  In particular, if the center were already known,
evaluation at $y=b$ would give
$d=D+\nabla\!\cdot u(b)+\|u(b)\|^2$.

The usable calibration does not differentiate the score.  Multiply the strong
identity by a continuously differentiable test function $\psi$ that is compactly
supported, or decays sufficiently fast for the boundary term to vanish.
Integration by parts gives
\begin{align}
 &\int\left[-\nabla\psi(y)^\top u(y)
 +\psi(y)\{\|u(y)\|^2+y^\top u(y)+D\}\right]\dd y\\
 &\hspace{35mm}=
 b^\top\int\psi(y)u(y)\dd y+d\int\psi(y)\dd y.
 \label{eq:weak-ou-proof}
\end{align}
Stacking these identities proves \eqref{eq:calibration-system}.  Full column
rank of $A$ makes its solution unique.

For completeness, let $\beta=(b,d)$ and suppose
$\widehat A=A+E$ remains full column rank.  Since $r=A\beta$,
\begin{equation}
 \widehat\beta-\beta
 =\widehat A^\dagger(\widehat r-\widehat A\beta)
 =\widehat A^\dagger(e-E\beta).
\end{equation}
Weyl's inequality gives
$\sigma_{\min}(\widehat A)\geq\sigma_{\min}(A)-\|E\|$; taking norms proves
\eqref{eq:calibration-bound}.  If $\widehat u$ has uniform score error at most
$\eps$ on the test supports, the error in each response is bounded by
\begin{equation}
 \|\nabla\psi\|_1\eps
 +\|\psi\|_1\left[(2U+Y)\eps+\eps^2\right],
 \label{eq:weak-score-error}
\end{equation}
where $U$ bounds $\|u\|$ and $Y$ bounds $\|y\|$ there.  The corresponding row
error is at most $\|\psi\|_1\eps$.  Numerical quadrature and tangent-approximation
errors enter the same $E,e$ terms, so the theorem applies without changing the
linear-algebra step.

The implementation uses normalized Gaussian tests
$\psi_{z,h}=\mathcal N(z,h^2I)$.  Dividing \eqref{eq:weak-ou-proof} by their
unit integral turns each row and response into
\begin{align}
 a_{z,h}&=(\E u(Y),1),\\
 r_{z,h}&=\E\left[
 h^{-2}(Y-z)^\top u(Y)+\|u(Y)\|^2+Y^\top u(Y)+D\right],
 \qquad Y\sim\mathcal N(z,h^2I),
 \label{eq:gaussian-weak-tests}
\end{align}
which requires only score values.  Gauss--Hermite quadrature evaluates these
expectations in the reported low-dimensional experiments.

Finally, suppose the cone measure and hence $q_1$ are invariant under
translation along a nonzero vector $v$.  Then $u(y)^\top v=0$ everywhere, so
every weak row annihilates $(v,0)$ and $A$ cannot have full column rank.  The
same invariance means that $b$ and $b+tv$ describe the same smoothed density.
The lost center coordinate is therefore not statistically identifiable.  A
full line is the simplest example; a genuine junction without lineality is
full rank for generic sufficiently rich test families.

\section{Tangent-score limit and shell identity}
\label{app:tangent}

The remaining appendix gives the full proof chain behind the shell inverse,
finite-query guarantees, counterexamples, and experiments.  We begin with the
forward tangent limit only to state the exact assumptions needed by the inverse
result; later sections then treat arbitrary-dimensional finite rays, the sharp
planar specialization, Hessian collisions, stability, finite data, scope, and
experimental protocols.

Let $T_{x_0,\sigma}(x)=(x-x_0)/\sigma$ and define the rescaled measure
\begin{equation}
 \nu_\sigma=\sigma^{-d}(T_{x_0,\sigma})_\#\mu.
\end{equation}
For a fixed normalized query $z$, put
\begin{equation}
 W_z(u)=e^{-\|z-u\|^2/2},\quad
 A_\sigma(z)=\int W_z(u)\dd\nu_\sigma(u),\quad
 B_\sigma(z)=\int(u-z)W_z(u)\dd\nu_\sigma(u).
 \label{eq:AB}
\end{equation}
Changing variables in the Gaussian convolution cancels all ambient Gaussian
normalization constants from the score ratio and gives the exact identity
\begin{equation}
 \sigma s_\sigma(x_0+\sigma z)=\frac{B_\sigma(z)}{A_\sigma(z)}.
 \label{eq:exactratio}
\end{equation}

Assume $\nu_\sigma$ converges to $\nu_\Lambda$ in the two Gaussian-weighted
moments in \eqref{eq:AB}, locally uniformly over the declared query set.  Since
$A_0(z)>0$, ratio continuity yields
\begin{equation}
 \frac{B_\sigma(z)}{A_\sigma(z)}\longrightarrow
 \frac{B_0(z)}{A_0(z)}=\nabla_z\log q_\Lambda(z),
 \qquad q_\Lambda(z)=\int W_z(u)\dd\nu_\Lambda(u).
\end{equation}
This proves \eqref{eq:tangentlimit}.  Notice that weak convergence by itself
is not the stated assumption; the Gaussian-weighted numerator and denominator
are the quantities actually needed.

\subsection{Finite-noise rate for planar branch junctions}

\paragraph{Idea.}
After rescaling by $\sigma$, each branch differs from its tangent ray by
$O(\sigma^\beta)$, and its density changes at the same order.  Gaussian weights
make these local errors integrable and suppress the separated remainder
exponentially; positivity keeps the final score ratio stable.

\begin{appendixtheorem}[Finite-noise planar branch junctions]
\label{thm:planar-tangent-bias}
Let $D=2$, $d=1$, and $0<\beta\leq1$.  Suppose that, near $x_0$, the measure is
a finite sum of half-branches $\gamma_j:[0,r_0]\to\R^2$ with densities
$\rho_j$ relative to arc length, plus a finite remainder supported a positive
distance from $x_0$.  Assume
\begin{equation}
 \gamma_j(0)=x_0,\quad \gamma_j'(0)=v_j,\quad \|v_j\|=1,
 \quad \|\gamma_j'(r)-v_j\|\leq L_\gamma r^\beta,
 \quad |\rho_j(r)-\rho_j(0)|\leq L_\rho r^\beta,
 \label{eq:planar-branch-class}
\end{equation}
where the $v_j$ are distinct and $\rho_j(0)>0$.  Then, for every compact
$Q\subset\R^2$ in the normalized coordinate $z=(x-x_0)/\sigma$, there are
$C_Q<\infty$ and $\sigma_0>0$ such that
\begin{equation}
 \sup_{z\in Q}\|\sigma s_\sigma(x_0+\sigma z)-F_\Lambda(z)\|
 \leq C_Q\sigma^\beta,
 \qquad 0<\sigma\leq\sigma_0,
 \label{eq:planar-tangent-score-rate}
\end{equation}
with $\Lambda=\sum_j\rho_j(0)\delta_{v_j}$.  In particular, $C^2$ branches
with $C^1$ positive densities give an $O(\sigma)$ bias.
\end{appendixtheorem}

After shortening the local
branch parameterizations if necessary, write the measure as
\begin{equation}
 \int f\dd\mu
 =\sum_{j=1}^{m}\int_0^{r_0}
 f(\gamma_j(r))\rho_j(r)\|\gamma_j'(r)\|\dd r
 +\int f\dd\mu_{\rm far},
 \label{eq:local-branch-decomposition}
\end{equation}
where $\operatorname{dist}(x_0,\operatorname{supp}\mu_{\rm far})\geq\Delta>0$.
Let $w_j=\rho_j(0)$ and
\begin{equation}
 \nu_0=\sum_{j=1}^{m}w_j\int_0^\infty\delta_{uv_j}\dd u,
 \qquad \Lambda=\sum_{j=1}^{m}w_j\delta_{v_j}.
 \label{eq:planar-tangent-measure}
\end{equation}

\begin{lemma}[Uniform Gaussian envelope]
\label{lem:gaussian-envelope}
Fix $R<\infty$.  Let $v$ be a unit vector, $u\geq0$, $\|z\|\leq R$, and
$q_t=uv+t\eta$ for $0\leq t\leq1$, where $\|\eta\|\leq u/2$.  There are
$c>0$ and $C_R<\infty$ such that
\begin{align}
 W_z(q_t)+\|\nabla W_z(q_t)\|
 &\leq C_R(1+u)e^{-cu^2},\label{eq:envelope-w}\\
 \|(q_t-z)W_z(q_t)\|+
 \|\nabla_q[(q-z)W_z(q)]_{q=q_t}\|_\op
 &\leq C_R(1+u^2)e^{-cu^2}.
 \label{eq:envelope-h}
\end{align}
\end{lemma}

\begin{proof}
The bounds $u/2\leq\|q_t\|\leq3u/2$ imply
$\|q_t-z\|\geq u/2-R$.  For $u\geq4R$, this is at least $u/4$; for
$0\leq u<4R$, every polynomial factor is bounded by a constant depending on
$R$.  The identities
\begin{equation*}
 \nabla_q W_z(q)=-(q-z)W_z(q),\qquad
 \nabla_q[(q-z)W_z(q)]
 =[I-(q-z)(q-z)^\top]W_z(q)
\end{equation*}
then prove both displays, for example with any fixed $c\leq1/64$ after
enlarging $C_R$.
\end{proof}

\begin{proof}[Proof of Theorem~\ref{thm:planar-tangent-bias}]
Define the density including the parameterization Jacobian by
\begin{equation}
 a_j(r)=\rho_j(r)\|\gamma_j'(r)\|.
 \label{eq:dressed-density}
\end{equation}
We may decrease $r_0$ until $L_\gamma r_0^\beta\leq1/2$; the omitted branch
segments can be absorbed into $\mu_{\rm far}$ while preserving a positive
distance from $x_0$.  Since
$|\|\gamma_j'(r)\|-1|\leq L_\gamma r^\beta$, the density assumptions give
constants $L_a,a_+<\infty$ such that
\begin{equation}
 |a_j(r)-w_j|\leq L_a r^\beta,
 \qquad 0\leq a_j(r)\leq a_+.
 \label{eq:dressed-density-rate}
\end{equation}
For example, if $\bar\rho=\max_{j,r}\rho_j(r)$ and
$\rho_+=\max_jw_j$, one may take
$L_a=\tfrac32L_\rho+\rho_+L_\gamma$ and $a_+=\tfrac32\bar\rho$.

For $0\leq u\leq r_0/\sigma$, define
\begin{equation}
 q_{j,\sigma}(u)=\frac{\gamma_j(\sigma u)-x_0}{\sigma}.
\end{equation}
Integrating the derivative estimate in \eqref{eq:planar-branch-class} yields
\begin{equation}
 \|q_{j,\sigma}(u)-uv_j\|
 \leq L_0\sigma^\beta u^{1+\beta},
 \qquad L_0=\frac{L_\gamma}{1+\beta}.
 \label{eq:rescaled-position-rate}
\end{equation}
The right-hand side is at most $u/2$.  Every segment from $uv_j$ to
$q_{j,\sigma}(u)$ therefore satisfies Lemma~\ref{lem:gaussian-envelope}.

Changing variables $r=\sigma u$ in \eqref{eq:local-branch-decomposition}
gives the local contributions
\begin{align}
 A_\sigma^{\rm loc}(z)
 &=\sum_{j=1}^{m}\int_0^{r_0/\sigma}
 a_j(\sigma u)W_z(q_{j,\sigma}(u))\dd u,
 \label{eq:local-a}\\
 B_\sigma^{\rm loc}(z)
 &=\sum_{j=1}^{m}\int_0^{r_0/\sigma}
 a_j(\sigma u)(q_{j,\sigma}(u)-z)
 W_z(q_{j,\sigma}(u))\dd u.
 \label{eq:local-b}
\end{align}
The tangent quantities replace $a_j(\sigma u)$ by $w_j$,
$q_{j,\sigma}(u)$ by $uv_j$, and the upper limit by infinity.

Let $Q$ be compact and $R_Q=\sup_{z\in Q}\|z\|$.  For the denominator
integrand, add and subtract $a_j(\sigma u)W_z(uv_j)$.  The mean-value theorem,
\eqref{eq:dressed-density-rate}, \eqref{eq:rescaled-position-rate}, and
Lemma~\ref{lem:gaussian-envelope} give, uniformly over $z\in Q$,
\begin{align}
 &|a_j(\sigma u)W_z(q_{j,\sigma}(u))-w_jW_z(uv_j)|\nonumber\\
 &\quad\leq C_{R_Q}\sigma^\beta
 \left[L_a u^\beta+a_+L_0u^{1+\beta}(1+u)\right]e^{-cu^2}.
 \label{eq:a-integrand-bound}
\end{align}
If $H_z(q)=(q-z)W_z(q)$, the same argument gives
\begin{align}
 &\|a_j(\sigma u)H_z(q_{j,\sigma}(u))-w_jH_z(uv_j)\|\nonumber\\
 &\quad\leq C_{R_Q}\sigma^\beta
 \left[L_a u^\beta(1+u)
 +a_+L_0u^{1+\beta}(1+u^2)\right]e^{-cu^2}.
 \label{eq:b-integrand-bound}
\end{align}
Both right-hand sides are integrable on $[0,\infty)$, so their integrals are
$O(\sigma^\beta)$.

The tangent tails beyond $r_0/\sigma$ have the same bound.  Indeed,
$W_z(uv_j)\leq e^{-(u-R_Q)^2/2}$ and
$\|H_z(uv_j)\|\leq(u+R_Q)e^{-(u-R_Q)^2/2}$.  A Gaussian tail times any fixed
polynomial is $O(\sigma^p)$ for every $p>0$ when its lower limit is
$r_0/\sigma$.

It remains to control $\mu_{\rm far}$.  Choose
$\sigma_0\leq\Delta/(2\max\{R_Q,1\})$.  For $z\in Q$ and
$x\in\operatorname{supp}\mu_{\rm far}$,
\begin{equation*}
 \left\|\frac{x-x_0}{\sigma}-z\right\|\geq\frac{\Delta}{2\sigma}.
\end{equation*}
The far contributions to $A_\sigma$ and $B_\sigma$ are consequently bounded by
constant multiples of
\begin{equation*}
 \sigma^{-1}\mu_{\rm far}(\R^2)e^{-\Delta^2/(8\sigma^2)}
 \quad\text{and}\quad
 \sigma^{-2}\mu_{\rm far}(\R^2)e^{-\Delta^2/(8\sigma^2)},
\end{equation*}
respectively.  Both are $O(\sigma^\beta)$ because, for $a,p>0$,
\begin{equation}
 \sup_{t>0}t^{-p}e^{-a/t^2}
 =a^{-p/2}\left(\frac{p}{2e}\right)^{p/2}<\infty.
 \label{eq:exponential-absorbs-power}
\end{equation}
Combining the local, tangent-tail, and far-field estimates proves
\begin{equation}
 \sup_{z\in Q}\{|A_\sigma(z)-A_0(z)|+
 \|B_\sigma(z)-B_0(z)\|\}
 \leq C_{\rm tan,Q}\sigma^\beta.
 \label{eq:planar-weighted-rate}
\end{equation}

Let $\rho_-:=\min_jw_j>0$.  Restricting each tangent integral to
$0\leq u\leq1$ gives the explicit lower bound
\begin{equation}
 a_Q:=\inf_{z\in Q}A_0(z)
 \geq m\rho_-e^{-(R_Q+1)^2/2}>0.
 \label{eq:planar-denominator-lower}
\end{equation}
If $C_{\rm tan,Q}\sigma^\beta\leq a_Q/2$, then $A_\sigma\geq a_Q/2$.  With
$S_Q=\sup_{z\in Q}\|B_0(z)/A_0(z)\|$, direct ratio subtraction yields
\begin{equation}
 \sup_{z\in Q}\left\|\frac{B_\sigma}{A_\sigma}
 -\frac{B_0}{A_0}\right\|
 \leq\frac{2(1+S_Q)C_{\rm tan,Q}}{a_Q}\sigma^\beta.
 \label{eq:planar-ratio-rate}
\end{equation}
The exact identity \eqref{eq:exactratio} proves
\eqref{eq:planar-tangent-score-rate}.  Projecting on $z=R\omega$ multiplies
the bound by at most $R$, so the same rate holds for the shell tangential score.
Finally, $C^2$ branches and $C^1$ densities satisfy
\eqref{eq:planar-branch-class} with $\beta=1$.
\end{proof}

The constant $C_{\rm tan,Q}$ depends on the compact query radius, branch count,
local chart radius, far-field distance and mass, the curve and density
H\"older bounds, and density upper bounds.  The ratio constant additionally
depends on positive tangent mass through $a_Q$.  A minimum angular separation
does not enter the forward approximation; it enters only the conditioning of
the downstream finite-atom inverse.

\paragraph{Numerical rate check.}
Figure~\ref{fig:tangent-bias-rate} evaluates the two Gaussian-weighted moments
and their ratio for the $\beta\in\{1/2,3/4,1\}$ branch/density cases in
Theorem~\ref{thm:planar-tangent-bias}.  The fitted small-$\sigma$ slopes track
the corresponding exponents;
this sweep checks the implementation and illustrates the asymptotic regime but
is not used in the proof.
\begin{figure}[H]
  \centering
  \includegraphics[width=0.99\linewidth]{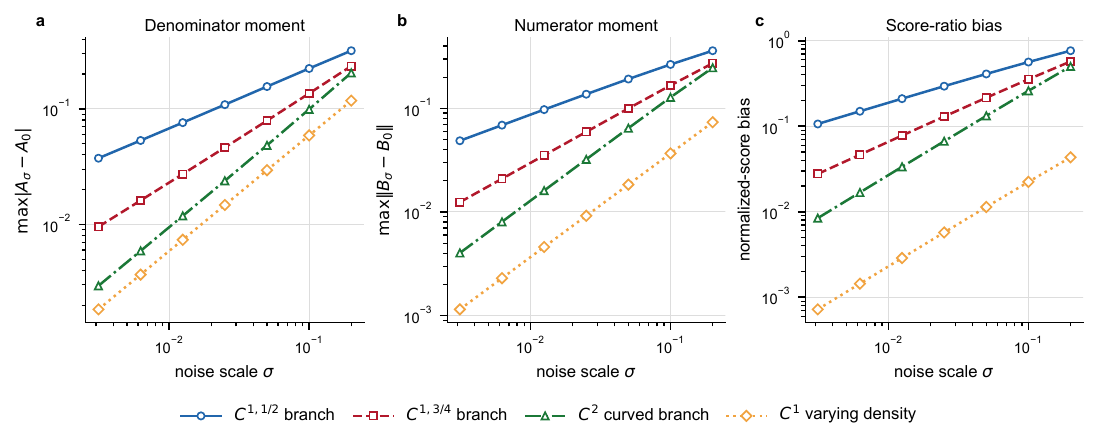}
  \caption{Finite-noise tangent-bias sweep.  The panels report the maximum
  denominator error, numerator error, and normalized-score ratio error over a
  fixed normalized query set.  Each curve is computed by Gaussian quadrature;
  the $\beta=1/2$, $\beta=3/4$, and $\beta=1$ cases exhibit their predicted
  log--log exponents.}
  \label{fig:tangent-bias-rate}
\end{figure}

\subsection{Shell identity}

For $z=R\omega$, substitute the polar form \eqref{eq:cone}:
\begin{align}
 q_\Lambda(R\omega)
 &=\int_{\Sph^{D-1}}\int_0^\infty
 e^{-(R^2+r^2-2Rr\omega^\top\theta)/2}
 r^{d-1}\dd r\,\dd\Lambda(\theta)\\
 &=e^{-R^2/2}\T_{d,R}\Lambda(\omega).
\end{align}
For any smooth scalar $f$ in a neighborhood of the sphere,
$\nabla_{\Sph^{D-1}} f(R\omega)=R P_{\omega^\perp}\nabla f(R\omega)$.
Applying this to $\log q_\Lambda$ proves \eqref{eq:shellidentity}.  The same
identity holds weakly and then smoothly after convolution when $\Lambda$ is a
finite measure rather than a density.

\section{Proof of one-shell injectivity}
\label{app:injectivity}

\begin{proof}[Proof of Theorem~\ref{thm:injective}]
Let $H=\T_{d,R}\Lambda$ and $H'=\T_{d,R}\Lambda'$.  The kernel is strictly
positive and the measures are nonzero, so both transforms are positive.
Equality of the observed fields gives
\begin{equation}
 \nabla_{\Sph^{D-1}}(\log H-\log H')=0.
\end{equation}
The sphere is connected for $D\geq2$, hence $H=cH'$ for some $c>0$.

It remains to prove that $\T_{d,R}$ has no harmonic nullspace.  Let
$Y_{\ell,m}$ be a degree-$\ell$ spherical harmonic.  The Funk--Hecke formula
for the exponential kernel gives
\begin{equation}
 \int_{\Sph^{D-1}}e^{Rr\omega^\top\theta}Y_{\ell,m}(\theta)\dd\theta
 =c_{D,\ell}(Rr)^{1-D/2}I_{\ell+D/2-1}(Rr)Y_{\ell,m}(\omega),
 \label{eq:funkhecke}
\end{equation}
where $c_{D,\ell}>0$ under the usual harmonic normalization and the modified
Bessel function $I_\nu(t)$ is strictly positive for $t>0$.  Integrating
\eqref{eq:funkhecke} against $r^{d-1}e^{-r^2/2}\dd r$ shows that every
degree-$\ell$ multiplier of $\T_{d,R}$ is strictly positive.  Therefore
$\T_{d,R}(\Lambda-c\Lambda')=0$ implies that every spherical-harmonic moment
of the finite signed measure $\Lambda-c\Lambda'$ vanishes.  Finite linear
combinations of spherical harmonics are uniformly dense in continuous
functions on the sphere, so the signed measure is zero.  Thus
$\Lambda=c\Lambda'$.  Unit mass forces $c=1$.
\end{proof}

The theorem is injective but not uniformly well-conditioned over arbitrary
angular measures.  The multipliers in \eqref{eq:funkhecke} can become small at
large $\ell$; finite-noise recovery must therefore restrict angular
complexity or explicitly pay the inverse multiplier.

\section{Finite positive rays in arbitrary dimension}
\label{app:finite-spherical}

We prove Theorem~\ref{thm:finite-spherical} and the general fixed-query lower
bound.  Let $\mathcal P_n(\Sph^{D-1})$ be the restrictions to the sphere of real
polynomials of total degree at most $n$.  Their harmonic decomposition is
\begin{equation}
 \mathcal P_n(\Sph^{D-1})
 =\bigoplus_{\ell=0}^{n}\mathcal H_\ell(\Sph^{D-1}).
 \label{eq:polynomial-harmonic-decomposition}
\end{equation}
\paragraph{Score-to-moment interface.}
Integrating the shell tangential score reconstructs the positive scalar
transform up to a constant; unit-mass normalization fixes it.  Dividing each
harmonic coefficient by the positive multiplier in \eqref{eq:funkhecke} recovers
\begin{equation}
 L(p)=\int_{\Sph^{D-1}}p(\theta)\dd\Lambda(\theta)
 \label{eq:polynomial-functional}
\end{equation}
for every $p\in\mathcal P_n$ when harmonics through degree $n$ are available.

\paragraph{Classical atom-recovery step.}
Write $\Lambda=\sum_{j=1}^{s}w_j\delta_{\theta_j}$ with $s\leq K$.  Degree-
$(s-1)$ Lagrange polynomials on the distinct support points show that evaluation
of $\mathcal P_{K-1}$ is surjective.  Thus, for a spanning vector $\Phi$ and
$V_{aj}=\phi_a(\theta_j)$, one has $\operatorname{rank}V=s$ and
\begin{align}
 H&:=L(\Phi\Phi^\top)=V\operatorname{diag}(w)V^\top,
 \qquad \operatorname{rank}H=s,\label{eq:general-moment-matrix}\\
 H_{B,r}&:=L(x_r\Phi_B\Phi_B^\top)
 =V_B\operatorname{diag}(w)\operatorname{diag}
 (\theta_{1r},\ldots,\theta_{sr})V_B^\top,
 \label{eq:shifted-moment-block}
\end{align}
where $\Phi_B$ selects any nonsingular $s\times s$ row submatrix $V_B$, and
$H_B:=L(\Phi_B\Phi_B^\top)=V_B\operatorname{diag}(w)V_B^\top$ is positive
definite.  Therefore
\begin{equation}
 H_B^{-1}H_{B,r}=V_B^{-\top}\operatorname{diag}
 (\theta_{1r},\ldots,\theta_{sr})V_B^\top.
 \label{eq:multiplication-pencil}
\end{equation}
These matrices commute, and their joint eigenvalue tuples are the directions
$\theta_j$; then $L(\Phi_B)=V_Bw$ recovers the weights.  Entries of $H$ and
$H_{B,r}$ have degree at most $2K-2$ and $2K-1$, respectively.  This standard
multivariate/spherical Prony step
\citep{kunis2016multivariate,kunis2019prony} proves the stated degree bound; the
preceding score-to-moment conversion is the paper-specific interface.

\begin{proof}[Proof of Proposition~\ref{prop:query-lower}]
Choose $K$ pairwise disjoint coordinate patches
$U_1,\ldots,U_K\subset\Sph^{D-1}$.  One direction in each patch contributes
$D-1$ Euclidean chart coordinates.  Parameterize the last normalized weight as
$w_K=1-\sum_{j<K}w_j$ and restrict the others to a nonempty open simplex.  The
resulting labeled family is an open subset of
$\R^{K(D-1)+K-1}=\R^{KD-1}$.

Any fixed collection of $M$ scalar tangential observations is a continuous map
from this open parameter set to $\R^M$.  If $M<KD-1$ and the observations
identified every measure, composing with the standard embedding
$\R^M\hookrightarrow\R^{KD-1}$ would give a continuous injection from an open
subset of $\R^{KD-1}$ into $\R^{KD-1}$ whose image lies in a lower-dimensional
coordinate subspace.  Invariance of domain says the image must be open, a
contradiction.  Therefore $M\geq KD-1$.
\end{proof}

\section{Exact recovery of finite positive planar rays}
\label{app:planar}

Identify $\Sph^1$ with angles $\phi\in[0,2\pi)$.  Let
$H(\phi)=\T_{1,R}\Lambda(\omega_\phi)>0$ and
$\bar H=(2\pi)^{-1}\int_0^{2\pi}H(\phi)\dd\phi$.  Equation
\eqref{eq:shellidentity} gives $(\log H)'$ exactly.  Integrating around the
circle determines $H$ up to a constant; the unit-mean density
$h=H/\bar H$ removes precisely the irrelevant total mass of $\Lambda$.

The circular Fourier expansion of the zonal kernel is
\begin{equation}
 K_{1,R}(\cos\phi)=\sum_{k\in\mathbb Z}\kappa_k(1,R)e^{ik\phi},\qquad
 \kappa_k(1,R)=c\int_0^\infty e^{-r^2/2}I_{|k|}(Rr)\dd r>0,
\end{equation}
with a common positive convention constant $c$.  Consequently the normalized
Fourier coefficients satisfy \eqref{eq:moments}.

\begin{proof}[Proof of Theorem~\ref{thm:planar}]
For $a=0,\ldots,K$ and $j=1,\ldots,s$, let
$V_{aj}=e^{-ia\theta_j}$.  Positivity gives
\begin{equation}
 T_K=V\operatorname{diag}(w)V^*,\qquad
 c^*T_Kc=\sum_{j=1}^{s}w_j\left|\sum_{a=0}^{K}c_ae^{ia\theta_j}\right|^2.
\end{equation}
The distinct-node Vandermonde matrix has full column rank, so
$\operatorname{rank}T_K=s$.  Because the weights are positive, a kernel vector
$c$ is equivalent to a degree-at-most-$K$ polynomial vanishing at every
$z_j=e^{-i\theta_j}$.  Conversely, $\prod_j(z-z_j)$ has degree $s\leq K$, so
the common unit-circle zeros of all kernel polynomials are exactly the nodes.
Finally, $m_0,\ldots,m_{s-1}$ form a nonsingular Vandermonde system for the
weights.
\end{proof}

\begin{proof}[Proof of Proposition~\ref{prop:sharp-moments}]
For $K\geq1$ and an angle $\alpha$, consider the normalized regular measure
\begin{equation}
 \Lambda_\alpha=\frac1K\sum_{j=0}^{K-1}
 \delta_{\alpha+2\pi j/K}.
\end{equation}
Its moments satisfy
\begin{equation}
 m_k(\Lambda_\alpha)
 =\frac{e^{-ik\alpha}}{K}\sum_{j=0}^{K-1}e^{-2\pi i k j/K}=0,
 \qquad 1\leq k<K,
\end{equation}
while $m_0=1$.  If $\alpha-\beta$ is not an integer multiple of $2\pi/K$,
then $\Lambda_\alpha\neq\Lambda_\beta$, although their moments through order
$K-1$ coincide.  Their next moments are
$m_K(\Lambda_\alpha)=e^{-iK\alpha}$ and
$m_K(\Lambda_\beta)=e^{-iK\beta}$, which are distinct.  This proves both the
failure below order $K$ and the stated sharpness.
\end{proof}

This proof also explains why positivity matters.  Without it, the quadratic
form need not be a sum of nonnegative terms, Toeplitz rank need not equal the
number of atoms, and cancellation can hide directions.

\section{Vertex Hessians and an infinite collision family}
\label{app:hessian}

Define the tilted probability law
\begin{equation}
 \Pi_z(\dd u)=q_\Lambda(z)^{-1}e^{-\|z-u\|^2/2}\nu_\Lambda(\dd u).
\end{equation}
Differentiation under the integral gives the standard Gaussian-mixture
identities
\begin{equation}
 F_\Lambda(z)=\E_{\Pi_z}[U]-z,\qquad
 \nabla F_\Lambda(z)=\operatorname{Cov}_{\Pi_z}(U)-I_D.
 \label{eq:posteriorhessian}
\end{equation}
At $z=0$, radius and angle are independent.  The radial density is
proportional to $r^{d-1}e^{-r^2/2}$, hence
\begin{equation}
 \E R=a_d=\sqrt2\frac{\Gamma((d+1)/2)}{\Gamma(d/2)},\qquad
 \E R^2=d.
\end{equation}
Therefore $\E U=a_dm_\Lambda$ and $\E UU^\top=dQ_\Lambda$.  Substitution into
\eqref{eq:posteriorhessian} proves \eqref{eq:hessian}.

\begin{proposition}[Regular-junction Hessian collision]
For every uniform regular planar $q$-ray angular measure with $q\geq3$,
$m_\Lambda=0$ and $Q_\Lambda=I_2/2$.  When $d=1$, every such junction has
$\nabla F(0)=-I_2/2$.
\end{proposition}
\begin{proof}
Write the directions as $e^{2\pi ij/q}$.  The first moment is the sum of the
$q$th roots of unity and vanishes.  Using
$\cos^2\theta=(1+\cos2\theta)/2$,
$\sin^2\theta=(1-\cos2\theta)/2$, and
$\sin\theta\cos\theta=\sin2\theta/2$, the second moment is $I_2/2$ because
the degree-two root sum vanishes for $q\geq3$.  Equation
\eqref{eq:hessian} completes the proof.
\end{proof}

Thus a one-point Hessian cannot identify even the branch count within this
simple infinite family.  By contrast, the regular $q$-ray angular measure has
a nonzero degree-$q$ moment, which survives the nonzero shell multiplier.

\section{Finite-query moment stability}
\label{app:finitequery}

We now prove \eqref{eq:momentbound}.  With the unit-mean $h$ defined above, let
\begin{equation}
 \ell(\phi)=\log h(\phi),\qquad g(\phi)=\ell'(\phi).
\end{equation}
By construction the continuous mean of $h$ is one.  Let
$\mathcal K_M=\{-\lfloor M/2\rfloor,\ldots,\lceil M/2\rceil-1\}$ and let
$a_M(n)\in\mathcal K_M$ be the representative congruent to $n$ modulo $M$.
Set
\begin{equation}
 \beta_M(n)=
 \begin{cases}
  1/a_M(n),&a_M(n)\neq0,\\
  0,&a_M(n)=0,
 \end{cases}
\end{equation}
where the second case records that spectral integration discards the aliased
zero mode.  If $g_n$ denotes the continuous Fourier coefficient, define
\begin{equation}
 A_M(g)=\sum_{n\notin\mathcal K_M}|g_n|
 \left|\beta_M(n)-\frac1n\right|.
 \label{eq:AM}
\end{equation}

Let $\widetilde y_k=M^{-1}\sum_jy_je^{-ik\phi_j}$ be the discrete Fourier
transform (DFT).  For nonzero
$k\in\mathcal K_M$, spectral integration sets
$\widetilde\ell_k=\widetilde y_k/(ik)$ and sets the zero coefficient to zero.
The sampling identity
\begin{equation}
 \widetilde g_k=\sum_{r\in\mathbb Z}g_{k+rM}
\end{equation}
shows that the noiseless nodewise integration error is bounded by
\eqref{eq:AM}.  For the noise, Parseval and Cauchy--Schwarz give
\begin{align}
 \max_j\left|\sum_{k\in\mathcal K_M\setminus\{0\}}
 \frac{\widetilde e_k}{ik}e^{ik\phi_j}\right|
 &\leq\left(\sum_{k\neq0}k^{-2}\right)^{1/2}
 \left(\sum_k|\widetilde e_k|^2\right)^{1/2}\\
 &\leq C_M\eps_2.
\end{align}
Thus, up to one additive log-density constant,
\begin{equation}
 \|\widetilde\ell-\ell\|_{\infty,M}\leq
 \tau_M=A_M(g)+C_M\eps_2.
 \label{eq:logerr}
\end{equation}
The bound does not grow with $M$.  If $g$ has $r>1$ bounded derivatives,
integration by parts also gives the explicit regularity estimate
\begin{equation}
 A_M(g)\leq \|g^{(r)}\|_\infty
 \sum_{n\notin\mathcal K_M}|n|^{-r}
 \left|\beta_M(n)-\frac1n\right|.
\end{equation}
For the positive analytic cone kernel all such derivatives exist, and they are
uniform over a compact declared geometry class.

Exponentiate the reconstructed log density at the nodes and normalize its
discrete mean.  Write $\bar h_M=M^{-1}\sum_jh(\phi_j)$.  The correct comparison
is with $h(\phi_j)/\bar h_M$, not with the continuously normalized
$h(\phi_j)$.  The additive log constant cancels, and the remaining discrete
normalization changes by at most a factor $e^{\tau_M}$ in either direction.
Thus \eqref{eq:logerr} implies
\begin{equation}
 \max_j\left|\frac{\widetilde h_j}{h_j/\bar h_M}-1\right|
 \leq e^{2\tau_M}-1.
 \label{eq:experr}
\end{equation}

The true normalized density has Fourier coefficients
$h_k=\rho_{|k|}(R)m_k$.  Sampling aliases them:
\begin{equation}
 \frac1M\sum_jh(\phi_j)e^{-ik\phi_j}
 =\rho_km_k+\sum_{r\neq0}\rho_{|k+rM|}m_{k+rM}.
\end{equation}
Because $\Lambda$ is a unit-mass positive measure, $|m_n|\leq1$.  Define
\begin{equation}
 B_{M,k}(R)=\sum_{r\in\mathbb Z\setminus\{0\}}\rho_{|k+rM|}(R).
\end{equation}
Then $B_{M,k}$ bounds the $k$th density alias uniformly over the angular
class.  Since $\bar h_M$ is the aliased discrete zero mode, it lies in
$[1-B_{M,0},1+B_{M,0}]$.  Combining this denominator perturbation with
\eqref{eq:experr}, and then dividing the $k$th mode by $\rho_k$, yields
\begin{equation}
 |\widehat m_k-m_k|\leq
 \rho_k^{-1}(e^{2\tau_M}-1)
 +\frac{B_{M,k}/\rho_k+B_{M,0}}{1-B_{M,0}},
\end{equation}
provided $B_{M,0}<1$.  This proves \eqref{eq:momentbound}.

\section{Rank, root, weight, and center stability}
\label{app:stability}

\begin{table}[H]
  \centering
  \small
  \caption{The conditional certificate proceeds from count to directions to
  weights.  Every row uses class information fixed before the field is
  inspected.  Failure stops the procedure before the next output.}
  \label{tab:certificate-stages}
  \begin{tabular}{@{}p{0.12\linewidth}p{0.25\linewidth}p{0.27\linewidth}p{0.25\linewidth}@{}}
    \toprule
    Stage & Declared class information & Sufficient test & Output on success / failure \\
    \midrule
    Count & Minimum positive Toeplitz eigenvalue & Signal gap exceeds twice $E_K$ & Branch count / abstain \\
    Directions & Count, angular separation, and root margin & Appendix Rouch\'e inequality & Angles within the declared tolerance / no directions \\
    Weights & Direction bound and minimum Vandermonde singular value & Appendix perturbed-solve inequality & Normalized weights within tolerance / no weights \\
    \bottomrule
  \end{tabular}
\end{table}

Let $\eta=\max_{1\leq k\leq K}|\widehat m_k-m_k|$ and set
$E_K=(K+1)\eta$.  Every entry of the Hermitian Toeplitz perturbation changes
by at most $\eta$, hence its Frobenius and operator norms obey
\begin{equation}
 \|\widehat T_K-T_K\|_\op\leq\|\widehat T_K-T_K\|_F\leq E_K.
\end{equation}
If $\gamma_K(\Lambda)=\lambda_{\min}^{+}(T_K)>2E_K$, Weyl's inequality
places all $s$ positive eigenvalues above $E_K$ and all $K+1-s$ zero
eigenvalues below $E_K$ after perturbation.  Thresholding at $E_K$ therefore
recovers the true rank.  For an honest uniform certificate, the declared
class must supply $\gamma_K\geq\Gamma_K>2E_K$.

The signal gap itself exposes two sources of difficulty.  Since
$T_K=V\operatorname{diag}(w)V^*$,
\begin{equation}
 \gamma_K(\Lambda)\geq w_{\min}\sigma_{\min}(V)^2.
 \label{eq:signalbound}
\end{equation}
If the minimum circular angle separation is $\delta_0$, the off-diagonal
Vandermonde Gram entries are Dirichlet sums bounded by
$\csc(\delta_0/2)$.  For modes $0{:}q$ and at most $u$ rays, Gershgorin gives
\begin{equation}
 \underline v_{q,u}^{2}
 :=q+1-(u-1)\csc(\delta_0/2)
 \leq \sigma_{\min}(V_{0:q})^2.
 \label{eq:declared-vandermonde}
\end{equation}
Whenever the right-hand side is positive and every branch weight is at least
$w_0$, \eqref{eq:signalbound} gives
$\gamma_q\geq w_0\underline v_{q,u}^{2}$.  The executable certificate uses
$(q,u)=(K,K)$ before rank is known, $(s,s)$ for root recovery, and $(K,s)$ for
the weight solve.  It also uses
$d_{\min}\geq2\sin(\delta_0/2)$ and the universal leading-coefficient bound
below.  If any conservative lower bound is nonpositive, it abstains.  Values
computed from the true synthetic geometry are retained only as oracle
conditioning diagnostics for the separate empirical sweeps.

For directions, form the $(s+1)\times(s+1)$ Toeplitz matrix $T_s$ and let
$c$ be its unit kernel vector, interpreted as the coefficients of the monic-up-
to-scale annihilating polynomial.  Put
\begin{equation}
 E_s=(s+1)\eta,\qquad
 \zeta=\frac{\sqrt2E_s}{\gamma_s-E_s},\qquad
 \gamma_s=\lambda_{\min}^{+}(T_s).
\end{equation}
When $E_s<\gamma_s$, the Davis--Kahan theorem permits a phase alignment such
that $\|\widehat c-c\|_2\leq\zeta$.  Let
  $d_{\min}=2\sin(\delta/2)$ and let $c_s$ be the leading coefficient.  For
  $s=1$, set $d_{\min}=2$ and interpret the empty product below as one; no
  separation assumption is needed.  On the
circle of radius $r<\min\{1,d_{\min}/2\}$ around a true root $z_j$,
\begin{equation}
 |P_c(z)|\geq |c_s|r(d_{\min}-r)^{s-1},\qquad
 |P_{\widehat c}(z)-P_c(z)|
 \leq\sqrt{s+1}(1+r)^s\zeta.
\end{equation}
Therefore the Rouch\'e condition
\begin{equation}
 \boxed{\sqrt{s+1}(1+r)^s\zeta
 <|c_s|r(d_{\min}-r)^{s-1}}
 \label{eq:rouche}
\end{equation}
places one estimated root in each disjoint disk.  Radial projection to the
unit circle gives the safe angular bound $2\arcsin r$.

Let this angular bound be $a$ and put
\begin{equation}
 L_{K,s}=\sqrt{s\sum_{k=0}^{K}k^2}.
\end{equation}
The leading-coefficient bound used by the executable certificate does not need
the unknown roots.  If $P_c(z)=c_s\prod_{j=1}^s(z-z_j)$ and $\|c\|_2=1$, then
the monic coefficient of degree $a$ has magnitude at most $\binom{s}{a}$,
because every $|z_j|=1$.  Vandermonde's identity therefore gives
\begin{equation}
 |c_s|^{-2}\leq\sum_{a=0}^s\binom{s}{a}^2
 =\binom{2s}{s},
 \qquad |c_s|\geq\binom{2s}{s}^{-1/2}.
 \label{eq:leading-coefficient-bound}
\end{equation}

The perturbed Vandermonde matrix satisfies
$\|\widehat V-V\|_\op\leq L_{K,s}a$.  Comparing the nonnegative
least-squares residual at its minimizer with the residual at the true weights
gives, before unit-mass renormalization,
\begin{equation}
 \|\widetilde w-w\|_2\leq
 \frac{2\{\sqrt{K+1}\eta+L_{K,s}a\}}
 {\sigma_{\min}(V)-L_{K,s}a},
 \label{eq:weightbound}
\end{equation}
provided the denominator is positive.  Unit-mass renormalization is stable
when the right-hand side, denoted $q_w$, is below $1/\sqrt{s}$.  Indeed,
$|\mathbf1^\top\widetilde w-1|\leq\sqrt{s}q_w$ and therefore
\begin{equation}
 \left\|\frac{\widetilde w}{\mathbf1^\top\widetilde w}-w\right\|_2
 \leq \frac{(1+\sqrt{s})q_w}{1-\sqrt{s}q_w}.
 \label{eq:normalizedweightbound}
\end{equation}
Equations
\eqref{eq:signalbound}, \eqref{eq:rouche}, and \eqref{eq:weightbound} are
sufficient certification tests when their required class-level bounds are
declared in advance.  Quantities evaluated using the true geometry are oracle
diagnostics, not data-derived certificates.

\subsection{Proof of the stagewise stability proposition}

We prove Proposition~\ref{prop:stagewise}.  Weyl's inequality and
$\|\widehat T_K-T_K\|_\op\leq(K+1)\eta$ show that
$\eta<\gamma_K/\bigl(2(K+1)\bigr)$ separates the $s$ positive eigenvalues from the
perturbed null eigenvalues.  This establishes the count stage without
solving for a single root or weight.

After count is known, Davis--Kahan controls the annihilating coefficient vector
by
\begin{equation}
 \|\widehat c-c\|_2\leq
 \frac{\sqrt2(s+1)\eta}{\gamma_s-(s+1)\eta}.
\end{equation}
Substitution into \eqref{eq:rouche}, with
$r\leq\sin(a_0/2)$, gives an additional sufficient condition for the declared
angular tolerance $a_0$.  It depends on the root separation $d_{\min}$ and the
annihilating polynomial, quantities absent from the count test.  Given the
angle bound, \eqref{eq:weightbound} adds the separate requirement
$L_{K,s}a<\sigma_{\min}(V)$ and a tolerance on the renormalized weight
error.  Hence this sufficient full-geometry certificate is the intersection of
the rank, root, and weight conditions, while the count certificate uses only
the first.  Whenever either later sufficient margin is smaller than
$\eta_{\rm count}$, the bounds exhibit an interval certified for count but not
for full geometry.  Failure of a later inequality is an abstention at that
stage, not evidence that the earlier recovered rank is wrong.

Finally, condition on a validated calibration error
$\|\widehat x_0-x_0\|=\sigma\|\widehat b-b\|$.  Querying around the estimated
center translates the normalized shell by $\widehat b-b$.  The mean-value
theorem gives
\begin{equation}
 |g_{\widehat b-b}(\phi)-g_0(\phi)|
 \leq R\|\widehat b-b\|
 \sup_{\|z-R\omega_\phi\|\leq\|\widehat b-b\|}
 \|\nabla F_\Lambda(z)\|_\op.
 \label{eq:centerbound}
\end{equation}
We denote the displayed supremum by $C_{\rm ctr}$.  This term is added before
any rank, root, or weight threshold is applied.

\section{Finite-data and learned-score errors}
\label{app:finitedata}

We give the constants behind \eqref{eq:endtoend}.  Let
$\mathcal N_{\rm ctr}=\bigcup_j\{z:\|z-z_j\|\leq\eps_{\rm ctr}\}$ contain
every translated query allowed by the calibrated center-error bound.  Assume
the weighted tangent
approximation
\begin{equation}
 \sup_{z\in\mathcal N_{\rm ctr}}\left\{|A_\sigma(z)-A_0(z)|+
 \|B_\sigma(z)-B_0(z)\|\right\}
 \leq C_{\rm tan}\sigma^\alpha
 \label{eq:tangentapprox}
\end{equation}
and define
\begin{equation}
 a_*:=\inf_{z\in\mathcal N_{\rm ctr}}A_0(z)>0,\qquad
 S_0:=\sup_{z\in\mathcal N_{\rm ctr}}\|B_0(z)/A_0(z)\|.
\end{equation}
If $C_{\rm tan}\sigma^\alpha\leq a_*/2$, direct ratio subtraction yields
\begin{equation}
 \max_j\|\sigma s_\sigma(x_j)-F_\Lambda(z_j)\|
 \leq\frac{2(1+S_0)C_{\rm tan}}{a_*}\sigma^\alpha
 =:\eps_{\rm geom}.
 \label{eq:geombound}
\end{equation}

For the empirical KDE, let
\begin{equation}
 W_{ij}=e^{-\|x_j-X_i\|^2/(2\sigma^2)},\qquad
 V_{ij}=\frac{X_i-x_j}{\sigma}W_{ij},
\end{equation}
with expectations $a_j=\E W_{ij}$, $b_j=\E V_{ij}$ and
$S=\max_j\|b_j/a_j\|$.  The envelopes
\begin{equation}
 0\leq W_{ij}\leq1,\qquad
 \|V_{ij}\|\leq e^{-1/2},\qquad
 \E\|V_{ij}\|^2\leq(2/e)a_j
 \label{eq:envelopes}
\end{equation}
follow from maximizing $re^{-r^2/2}$ and using
$r^2e^{-r^2}\leq(2/e)e^{-r^2/2}$.

Centering gives
$\|V_{ij}-b_j\|\leq2e^{-1/2}$ and
$\E\|V_{ij}-b_j\|^2\leq(2/e)a_j$.  We use the Hilbert-space Bernstein
consequence of \citet[Theorem~3.4]{pinelis1994martingales}:
\begin{equation}
 \left\|\frac1N\sum_{i=1}^NY_i\right\|
 \leq\sqrt{\frac{2vt}{N}}+\frac{2Lt}{3N}
 \label{eq:hilbertbernstein}
\end{equation}
with probability at least $1-2e^{-t}$ for independent mean-zero vectors with
$\|Y_i\|\leq L$ and $\E\|Y_i\|^2\leq v$.  Set $t=\log(4M/\delta)$ and
\begin{equation}
 D_j=\sqrt{\frac{2a_jt}{N}}+\frac{t}{3N},\qquad
 Q_j=\sqrt{\frac{4a_jt}{eN}}+
 \frac{4e^{-1/2}t}{3N}.
\end{equation}
Scalar Bernstein for $W_{ij}$, vector Bernstein for $V_{ij}$, and a union
bound imply that, with probability at least $1-\delta$, simultaneously
\begin{equation}
 |\overline W_j-a_j|\leq D_j,\qquad
 \|\overline V_j-b_j\|\leq Q_j.
\end{equation}
Whenever $D_j<a_j$, another ratio subtraction gives
\begin{equation}
 \left\|\frac{\overline V_j}{\overline W_j}-\frac{b_j}{a_j}\right\|
 \leq\frac{Q_j+SD_j}{a_j-D_j}.
 \label{eq:ratioconc}
\end{equation}

Suppose $a_j\geq a_-\sigma^d$ and put
$\lambda=Na_-\sigma^d$.  If
\begin{equation}
 \sqrt{\frac{2t}{\lambda}}+\frac{t}{3\lambda}\leq\frac12,
 \label{eq:coveragecondition}
\end{equation}
then \eqref{eq:ratioconc} simplifies to
\begin{equation}
 \eps_{\rm KDE}=2\left[
 \sqrt{\frac{4t}{e\lambda}}+
 \frac{4e^{-1/2}t}{3\lambda}
 +S\left(\sqrt{\frac{2t}{\lambda}}+
 \frac{t}{3\lambda}\right)\right].
 \label{eq:explicitkde}
\end{equation}
This makes the denominator condition and the effective local coverage
explicit, rather than hiding them inside big-$O$ notation.

If the learned score has validated pointwise error
\begin{equation}
 \max_j\|s_\theta(x_j,\sigma)-\widehat s_{N,\sigma}(x_j)\|
 \leq\eps_\theta(\sigma),
\end{equation}
its normalized vector-field contribution is $\sigma\eps_\theta$.  Because
$g(\phi)=R\langle t_\phi,F_\Lambda(R\omega_\phi)\rangle$, converting a
normalized vector-field error into the scalar shell observation multiplies it
by $R$.  Adding \eqref{eq:geombound}, \eqref{eq:explicitkde}, the network term,
and \eqref{eq:centerbound} therefore proves \eqref{eq:endtoend}, with
$\eps_2\leq\eps_{\rm shell}$.  If $d$ is replaced by an estimate
$\widehat d$, fix a compact declared interval $[d_-,d_+]\subset(0,\infty)$
containing both values.  Differentiation under the radial integral shows that
each multiplier $\kappa_k(d,R)$ used by the finite inverse is continuously
differentiable in $d$.  Positivity and compactness therefore give
\begin{equation}
 \max_{1\leq k\leq K}
 \left|\log\rho_k(\widehat d,R)-\log\rho_k(d,R)\right|
 \leq C_d|\widehat d-d|.
 \label{eq:dimension-multiplier-error}
\end{equation}
Since $|m_k|\leq1$, this changes the recovered moments by at most a constant
multiple of $\eps_d$.  Inserting the combined result into
the finite-query theorem and then the three perturbation tests proves the
end-to-end guarantee with no additional probability loss.

\section{Counterexamples delimiting the observation model}
\label{app:counterexamples}

The following examples separate limitations of particular detectors from the
information available in the full score field.  They are not assumptions or
ingredients of the one-shell theorem.

\subsection{Fine-scale Hessian collapse with retained score displacement}

For an empirical Gaussian KDE, let $\pi_i(x)$ be the Gaussian posterior weight
of sample $X_i$.  Direct differentiation gives
\begin{equation}
 \widehat s_{N,\sigma}(x)
 =\frac{\sum_i\pi_i(x)X_i-x}{\sigma^2},
 \qquad
 \nabla^2\log\widehat p_{N,\sigma}(x)
 =-\frac{I}{\sigma^2}+
 \frac{\operatorname{Cov}_{\pi(x)}(X)}{\sigma^4}.
 \label{eq:empiricalidentities}
\end{equation}
If $X_{i_*}$ is the unique nearest sample to $x$, every other posterior ratio
decays as $\exp(-c_i/\sigma^2)$ for some $c_i>0$.  Hence
\begin{equation}
 \sigma^2\nabla^2\log\widehat p_{N,\sigma}(x)\longrightarrow-I,
 \qquad
 \sigma^2\widehat s_{N,\sigma}(x)\longrightarrow X_{i_*}-x.
 \label{eq:nearest}
\end{equation}
At very fine noise, a Hessian-sign detector loses posterior-covariance
information, whereas the full score retains the nearest-sample displacement.

\subsection{The crossing/gap critical scale is detector- and measure-specific}

For $t\in[-L,L]$ and half-gap $\Delta>0$, consider
\begin{equation}
 q_\pm^0(t)=(t,\pm|t|),\qquad
 q_\pm^\Delta(t)=\left(t,\pm\sqrt{t^2+\Delta^2}\right).
 \label{eq:crossgap}
\end{equation}
We compare parameter measure $\dd t$ with arc-length measure
$w_\Delta(t)\dd t$, where
\begin{equation*}
 w_\Delta(t)=\sqrt{\frac{2t^2+\Delta^2}{t^2+\Delta^2}}.
\end{equation*}
At the symmetric origin, the infinite parameter-measure control gives
\begin{equation}
 H_{\rm cross}(0)=-\frac1{2\sigma^2}I_2,\qquad
 H_{\rm gap,param}(0)=
 \begin{pmatrix}
 -\frac1{2\sigma^2}&0\\[2pt]
 0&-\frac1{2\sigma^2}+\frac{\Delta^2}{\sigma^4}
 \end{pmatrix}.
\end{equation}
Its separation-direction eigenvalue changes sign at
$\sigma_c^{\rm param}=\sqrt2\Delta$.

Arc length changes the constant.  Put $r=\Delta/\sigma$ and
\begin{equation}
 m_2(r)=\frac{\int_\R u^2
 \sqrt{(2u^2+r^2)/(u^2+r^2)}e^{-u^2}\dd u}
 {\int_\R \sqrt{(2u^2+r^2)/(u^2+r^2)}e^{-u^2}\dd u}.
\end{equation}
The dimensionless separation-direction eigenvalue is
$m_2(r)+r^2-1$.

\begin{proposition}[Unique arc-length critical scale]
Writing $s=r^2$ and $F(s)=m_2(\sqrt s)+s-1$, $F$ has one positive,
transverse root in $(\sqrt2-1,1/2)$.  Deterministic quadrature gives
\begin{equation}
 \frac{\sigma_c^{\rm arc}}{\Delta}=\ArcCriticalRatio.
\end{equation}
\end{proposition}
\begin{proof}
Let $\pi_s$ have density proportional to
$\sqrt{(2t^2+s)/(t^2+s)}e^{-t^2}$ and write $m(s)=\E_s[t^2]$.
Integration by parts yields
\begin{equation}
 2m(s)=1+\E_s\frac{s t^2}{(2t^2+s)(t^2+s)},
\end{equation}
so $1/2<m(s)<2-\sqrt2$.  These bounds give opposite signs at
$s=\sqrt2-1$ and $s=1/2$.  On that interval, differentiating under the
integral and applying Popoviciu and Cauchy--Schwarz gives
\begin{equation}
 F'(s)\geq
 1-\frac{\sqrt2-1}{4}\sqrt{\frac{3\sqrt2}{4}}
 >\ArcDerivativeLower.
\end{equation}
Thus the root exists, is unique, and is transverse.  The displayed decimal is
generated by deterministic quadrature.
\end{proof}

\subsection{The coarse exponent depends on the reference measure}

Let $L=\ell\sigma$, $r=\Delta/\sigma$, and define the physical score
discrepancy on $\|x\|\leq c\sigma$ by
\begin{equation}
 D_{\sigma,L}^{\mu}(c)
 :=\sup_{\|x\|\leq c\sigma}
 \|s_{\sigma,L,\Delta}^{\mu}(x)-s_{\sigma,L,0}^{\mu}(x)\|,
\end{equation}
where $\mu$ is either parameter measure or arc length.

\begin{proposition}[Measure-dependent coarse-scale laws]
For fixed $c,\ell>0$ and $r\downarrow0$, finite positive constants satisfy
\begin{equation}
 \begin{aligned}
 D_{\sigma,\ell\sigma}^{\rm arc}(c)
 &=C_{\rm arc}(c,\ell)\frac{\Delta}{\sigma^2}
   +o\!\left(\frac{\Delta}{\sigma^2}\right),\\
 D_{\sigma,\ell\sigma}^{\rm param}(c)
 &=C_{\rm param}(c,\ell)\frac{\Delta^2}{\sigma^3}
   +o\!\left(\frac{\Delta^2}{\sigma^3}\right).
 \end{aligned}
 \label{eq:coarselaws}
\end{equation}
\end{proposition}
\begin{proof}
After reflection symmetry is used, the rescaled geometric kernel is analytic in
$r^2$.  Parameter measure therefore changes its density by
$r^2J+o_{C^1}(r^2)$.  For arc length,
\begin{equation}
 w_r(u)-\sqrt2=h(u/r),\qquad
 h(v)=\sqrt{\frac{2v^2+1}{v^2+1}}-\sqrt2\in L^1(\R),
\end{equation}
and $\int h\neq0$.  This width-$r$ boundary layer changes the density by
$rH+o_{C^1}(r)$.  A physical score is $\sigma^{-1}\nabla_z\log P$, producing
the two orders in \eqref{eq:coarselaws}.  The leading ratios $H/P_0$ and $J/P_0$
are nonconstant on coordinate lines, so both displayed constants are positive.
\end{proof}
The linear arc-length term comes from the mass element, not from branch
displacement.  Thus neither $(\Delta/\sigma)$ nor $(\Delta/\sigma)^2$ is a
measure-independent score law; physical score units add another $1/\sigma$.

\subsection{Why these facts do not prove multiscale necessity}

At the common center, $g(y)=y_2^2-y_1^2$ is zero on the crossing and equals
$\Delta^2$ on the gap.  By \eqref{eq:nearest}, a sufficiently fine raw-score
query recovers a nearest sample and distinguishes these supports.  For a class
with $\Delta\geq\Delta_{\min}>0$, one common sufficiently fine scale works.
The Hessian detector's bandwidth window is therefore not a fixed-scale lower
bound for the full collection of score queries.

Small far-field density is also irrelevant to log-score magnitude.  For a
point mass,
\begin{equation}
 p_\sigma(x)\propto e^{-\|x-y\|^2/(2\sigma^2)},\qquad
 \nabla\log p_\sigma(x)=\frac{y-x}{\sigma^2}.
\end{equation}
A deterministic check reaches
$\log_{10}p_\sigma=\FarFieldLogDensity$ while
$\|s_\sigma\|=\FarFieldScoreNorm$.  Density vanishes exponentially while the
score grows with distance.

Consequently, heterogeneous feature sizes alone do not imply multiscale
necessity under an absolute-error score oracle.  The main result instead asks
what one shell identifies.

\section{Experimental protocols and additional diagnostics}
\label{app:experiments}

\begin{table}[H]
  \centering
  \small
  \caption{Information permissions for the three reported inference roles.
  Only a certified output carries a sufficient guarantee, and only when its
  external error budget and declared class bounds are valid.}
  \label{tab:inference-roles}
  \begin{tabular}{@{}p{0.18\linewidth}p{0.31\linewidth}p{0.39\linewidth}@{}}
    \toprule
    Role & Information beyond the queried field & Interpretation \\
    \midrule
    Blind candidate & $R$, $K$, query grid, and numerical rank rule & Uncertified count, directions, and weights \\
    Oracle diagnostic & True synthetic signal gap or geometry & Controlled diagnosis only; not deployable \\
    Certified output & Valid moment-error budget and class bounds fixed in advance & Stagewise guarantee, otherwise abstention \\
    \bottomrule
  \end{tabular}
\end{table}

\subsection{Single-scale calibration and three-dimensional recovery}

The weak-calibration sweep draws twenty separated positive ray cones for each
$D\in\{2,3,5\}$.  It uses $3(D+1)$ Gaussian test functions with normalized
bandwidth $h=0.75$ and tensor Gauss--Hermite rules of orders 15, 13, and 8,
respectively.  Each exact score is also perturbed by a deterministic smooth
random field at root-mean-square magnitudes
$0.003$, $0.01$, and $0.03$.  Test centers, cone centers, directions, weights,
and perturbations are all generated from recorded integer seeds.  The table
reports exact numerical quadrature error separately from perturbation
sensitivity.

The spherical experiment alternates two and three non-coplanar rays over
twenty seeds.  At normalized radius $R=2$, it reconstructs log shell density by
24-point Gauss--Legendre integration along each north-pole meridian on a
$24\times56$ polar--azimuth quadrature grid.  A 96-point rule computes the
Gaussian--cone multipliers.  Spherical harmonics through degree five are
converted to polynomial moments, after which a pivoted moment block and three
coordinate multiplication matrices recover the joint eigenvalue tuples.  Full
geometry requires the correct count, maximum angular error at most
$0.1^\circ$, and maximum weight error at most $10^{-3}$.

The learned calibration uses a non-coplanar weighted three-ray distribution in
$\R^3$ translated away from the coordinate origin.  Each of five seeds samples
50,000 points on length-four rays and trains a four-block residual SiLU MLP of
width 192 for 8,000 AdamW updates with batches of 1,024.  Denoising score
matching draws noise levels $0.12$, $0.18$, and $0.25$; calibration is evaluated
at $\sigma=0.18$ with the same weak Gaussian tests.  The implementation records
one aggregate row per seed and discards the model immediately after evaluation.
It writes no checkpoint, score-prediction array, run manifest, or hash ledger.

\subsection{Shell feasibility and finite-query evaluation}

The deterministic shell-feasibility experiment verifies the regular
three-ray/four-ray Hessian collision, computes their shell-field separation,
and scans the inverse harmonic multipliers over normalized radius and degree.
The finite-query matrix then varies
\begin{equation}
 \begin{aligned}
 R&\in\{2,3,4\},\\
 M&\in\{12,16,24,32,48,64\},\\
 \text{score-noise standard deviation}
   &\in\{0,0.001,0.01,0.05\}.
 \end{aligned}
\end{equation}
over seeds $0{:}11$, angular separations from $10^\circ$ to $90^\circ$, and
minimum weights from 0.02 to 0.25.  It includes two-, three-, and four-ray
cases, an isolated close pair, a four-ray cluster, blind and oracle rank, and
normalized center offsets through 0.8.  These crossed factors produce
\FiniteQueryTrials{}
trials.  Exact evaluation cases at $R=3,M=32$ all recover.  The noisy results are
reported as an empirical conditioning map, not as a replacement for the
worst-case perturbation theorem.

Every numerical candidate uses $K=4$.  The blind rank rule sorts the five
Toeplitz eigenvalues in decreasing order and maximizes the adjacent relative
gap
\begin{equation}
 r_s=\frac{\max\{\lambda_s,\epsilon_{\rm eig}\}}
 {\max\{|\lambda_{s+1}|,\epsilon_{\rm eig}\}},\qquad s=1{:}4,
\end{equation}
after flooring magnitudes at
\begin{equation}
 \epsilon_{\rm eig}=\max\{10(K+1)\epsilon_{\rm mach}\lambda_{\max},10^{-15}\}.
\end{equation}
This relative-gap rule produces only an uncertified candidate.  Whenever a
valid moment-error bound $\bar\eta$ and a declared class gap are supplied, the
certified path instead sets $E_K=(K+1)\bar\eta$ and counts the eigenvalues
strictly above $E_K$.  It then verifies that the declared gap exceeds $2E_K$
and abstains if this test fails.  It never substitutes the largest internal
signal eigengap for the zero-versus-positive threshold.
For the selected rank $s$, the smallest-eigenvalue eigenvector of $T_s$
supplies the annihilating coefficients.  Reversing these coefficients exposes
the roots, which are projected radially to the unit circle.  Nonnegative least
squares on the real and imaginary parts of modes $0{:}K$ then gives the
normalized weights.  The reported total comprises
\FiniteConditioningTrials{} conditioning trials and
\FiniteLocalizationTrials{} artificial center-offset trials.

The separate certificate execution fixes its geometry class before any
perturbation: $K=4$, $w_0=0.10$, $\delta_0=80^\circ$, angle tolerance
$5^\circ$, and weight tolerance $0.05$.  It includes regular and unequal
three- and four-ray measures, all satisfying those bounds.  For each of seven
moment-error budgets and twelve seeds, every nonzero moment receives a random
complex perturbation whose magnitude is at most the stated budget.  A seed
keeps the same perturbation direction while its magnitude changes across error
budgets, producing paired error paths rather than independent trials.  The code
uses the worst-case Gershgorin bounds in
\eqref{eq:declared-vandermonde}, the coefficient bound
\eqref{eq:leading-coefficient-bound}, and bisection for the smallest radius
satisfying the Rouch\'e test.  It then applies the normalized weight-error
bound.  This produces \CertificateExecutionTrials{} trials.  Automated checks
verify that every count report is correct and every reported direction or
weight error lies below its computed bound.  The implementation assumes a
valid input error budget and standard floating-point arithmetic; its numerical
floor protects linear algebra but is not an interval-arithmetic proof.

\begin{table}[H]
  \centering
  \caption{Execution of the pre-specified count--direction--weight certificate.
  Entries are positive reports out of 48 bounded perturbations at each error
  level; unreported cases abstain at that stage.}
  \label{tab:certificate-execution}
  \begin{tabular}{rccc}
\toprule
Moment-error bound & Count & Directions & Full geometry \\
\midrule
$1e-09$ & 48/48 & 48/48 & 48/48 \\
$1e-08$ & 48/48 & 48/48 & 48/48 \\
$1e-07$ & 48/48 & 48/48 & 48/48 \\
$1e-06$ & 48/48 & 48/48 & 24/48 \\
$1e-04$ & 48/48 & 24/48 & 0/48 \\
$1e-03$ & 48/48 & 0/48 & 0/48 \\
$1e-02$ & 0/48 & 0/48 & 0/48 \\
\bottomrule
\end{tabular}

\end{table}

\begin{figure}[H]
  \centering
  \includegraphics[width=0.99\linewidth]{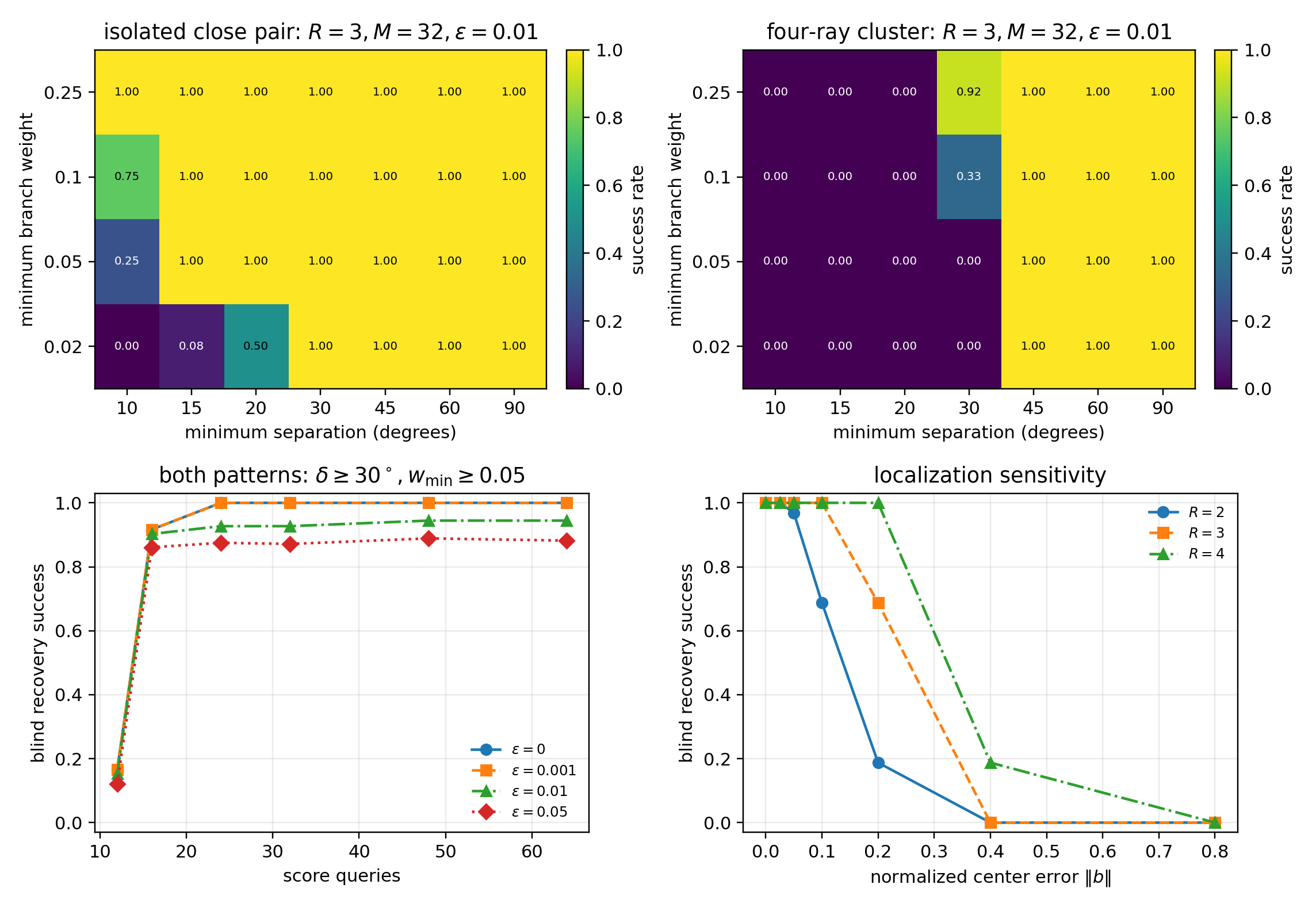}
  \caption{Finite-query recovery separates conditioning, query, noise, and
  center-offset failure modes.  Rates aggregate 12 seeds per configuration;
  success requires correct count, maximum angular error $5^\circ$, and maximum
  weight error $0.05$.  Top: an isolated close pair and a four-ray cluster at
  matched minimum separation and weight.  Bottom: query/noise sensitivity and
  center-offset sensitivity by shell radius.  Cell numbers, markers, and line
  styles make the rates readable without relying on color.}
  \label{fig:finite-query-diagnostics}
\end{figure}

\subsection{Population, KDE, held-out geometry, and perturbations}

The named geometry set contains an endpoint, a line, two corners, a regular
Y, a crossing, an unequal four-ray junction, and a clustered four-ray
junction.  Twelve held-out instances are deterministically generated with two
to four rays, continuous angles and weights, and minimum separation
$30^\circ$; they are not selected from the named template library.  Radial
Poisson coverage ranges from 16 to 1024 over seeds $0{:}31$.

For each sample, the script computes the exact population score, the exact
empirical Gaussian KDE score, blind recovery, the fixed-template
baseline, and a one-point Hessian baseline based on principal-component
analysis (PCA).  The full target is count,
angles, and weights rather than a geometry label.  In normalized coordinates,
a curved branch is parameterized as
$r\theta+(0.02/2)r^2\theta^\perp$ for $0\leq r\leq9$; ambient perturbations add
isotropic Gaussian noise of coordinate standard deviation 0.1, and center
error translates every query by $(0.1,0)$.  Additional matrices apply each
perturbation and all three together at coverages 256 and 1024.  Twice the
split-sample moment discrepancy plus the exactly computed population-to-tangent
moment bias feed an oracle rank gate, which compares
that proxy with half the true synthetic Toeplitz signal gap.  This is neither a
calibrated high-probability bound nor an implementation of the later root and
weight tests.  We therefore report its abstention and its full-recovery rate
conditional on reporting a rank.

The \EmpiricalKDETrials{} empirical trials comprise \CoverageTrials{} coverage
trials and \PerturbationTrials{} perturbation trials.  In the coverage matrix,
\CoverageNonAbstaining{} fields are non-abstaining and
\CoverageCorrectReported{} of those meet the full-geometry tolerance.  In the
perturbation matrix, the diagnostic reports a rank for
\PerturbationNonAbstaining{} fields, and all \PerturbationCorrectReported{}
meet the same tolerance.  These oracle-assisted conditional counts are kept
separate from blind numerical recovery.

\begin{figure}[H]
  \centering
  \includegraphics[width=0.98\linewidth]{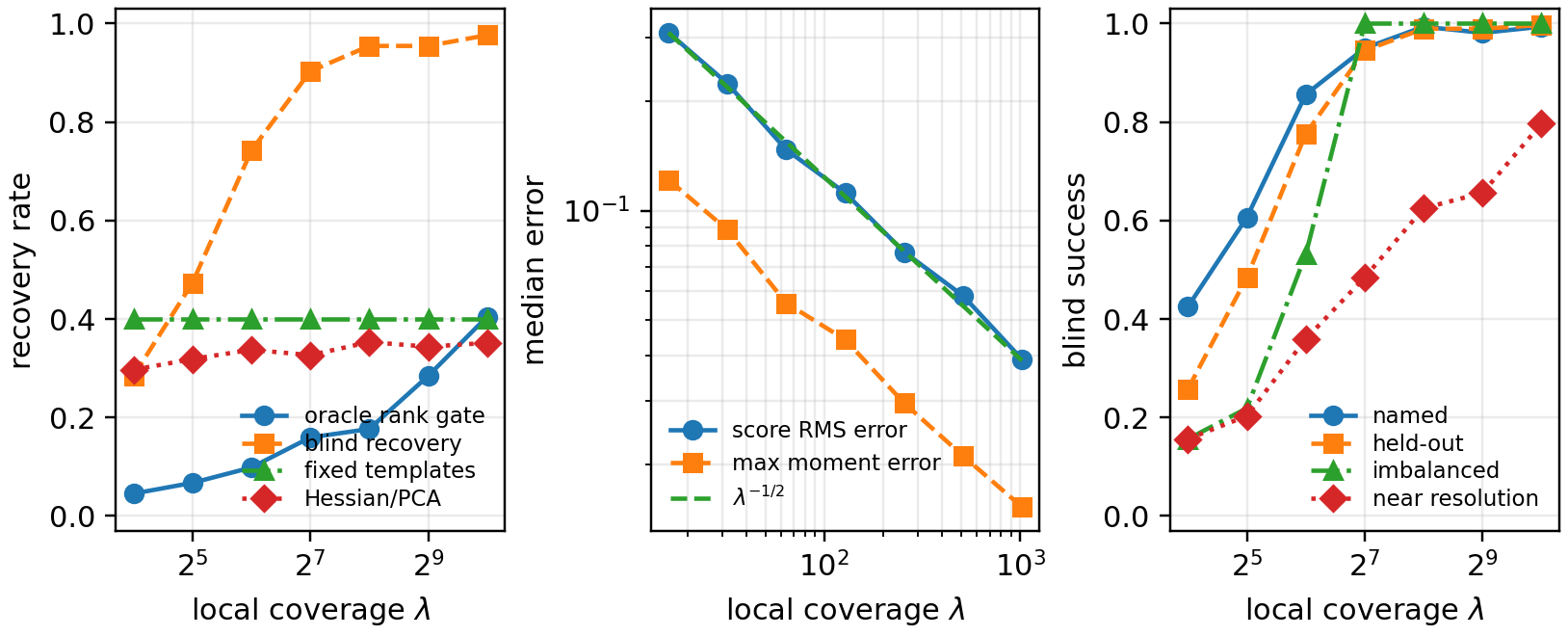}
  \caption{Empirical KDE recovery approaches the population field as local
  coverage grows.  Across \EmpiricalKDETrials{} trials, local coverage is the
  expected effective number of nearby samples receiving appreciable Gaussian
  weight at a query.  Left: oracle-rank and blind recovery, with fixed-template
  and one-point Hessian/PCA checks; the rank gate uses the true synthetic signal
  gap.  Middle: normalized-score RMS, maximum moment error, and an inverse-square-root
  reference slope.  Right: blind recovery on named, held-out, imbalanced, and
  near-resolution subsets.  Markers and line styles identify every curve
  without relying on color.}
  \label{fig:kde}
\end{figure}

\FloatBarrier

\subsection{Single-geometry training-budget sensitivity}

The sensitivity study fixes one weak four-ray geometry: angles
$(0.05,1.18,3.02,5.10)$ and weights $(0.05,0.25,0.30,0.40)$.  It crosses three
dataset sizes (512, 2,048, and 8,192), three noise scales (0.1, 0.2, and 0.4),
widths 32 and 128, \LearnedSeedCount{} seeds, batch size 512, and 1,000 updates.
We evaluate a time-conditioned network and single-noise specialists, extending
four width-128 controls to 5,000 updates.

At the standard budget, \SensitivityUndercountSettings{} configurations had
empirical KDE full success in every seed but learned undercount in
\SensitivityUndercountSeeds{} seeds.  The \SensitivityLongControls{} long
controls then restored count in several settings; full weighted geometry
remained outside tolerance in
\SensitivityLongFailureMin{}--\SensitivityLongFailureMax{} of seeds.  These
runs do not support irreversible branch loss as the cause of the earlier count
errors.  This single-geometry, single-MLP matrix measures training-budget
sensitivity; the fixed evaluation below covers multiple geometries and
parameterizations.

\begin{figure}[H]
  \centering
  \includegraphics[width=0.99\linewidth]{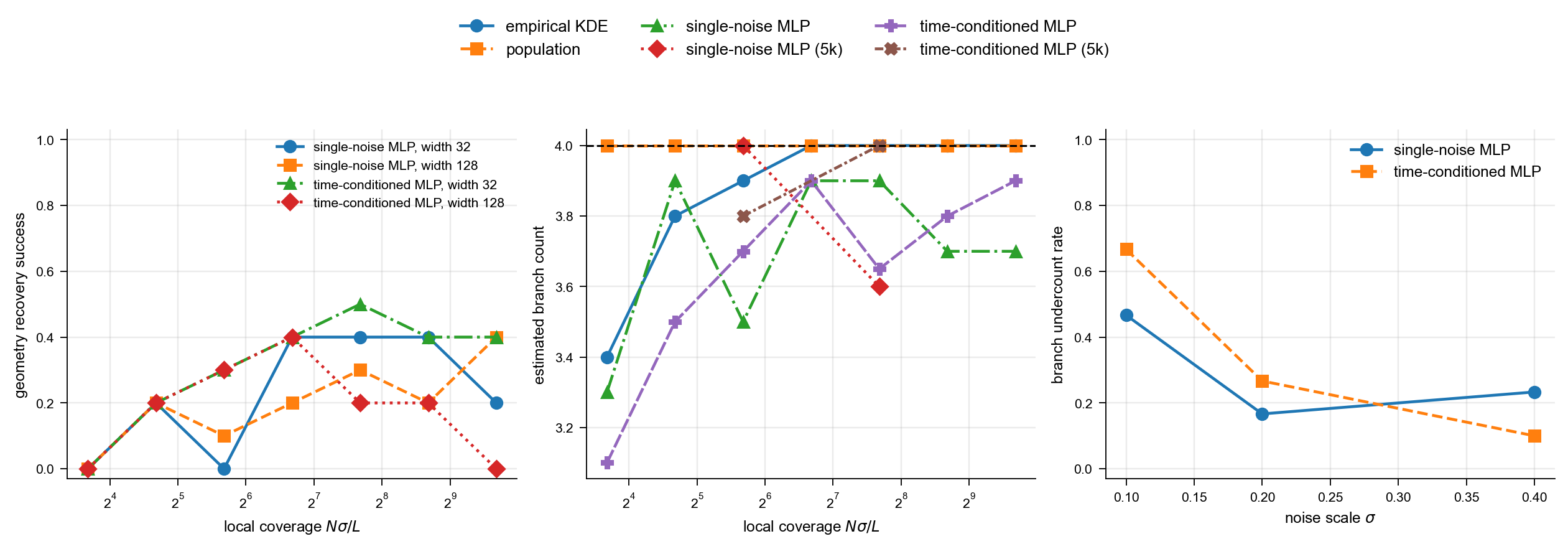}
  \caption{Longer training can restore count without full geometry.  The study
  uses \LearnedSeedCount{} seeds per configuration; its local-coverage proxy is
  $N\sigma/L$.  Left: full-geometry success versus coverage.  Middle: estimated
  count, with the true count marked by a dashed line.  Right: undercount rate
  versus noise scale.  Markers and line styles distinguish empirical,
  time-conditioned, and single-noise fields and their training budgets.  The
  external legend identifies the six middle-panel fields; the other panels use
  local legends.}
  \label{fig:training-sensitivity}
\end{figure}

\subsection{Cross-geometry evaluation}

The fixed evaluation matrix uses the following geometries:
\begin{table}[H]
\centering
\caption{Synthetic geometries used in the fixed cross-geometry learned-score
evaluation.}
\label{tab:learned-geometries}
\begin{tabular}{lll}
\toprule
Geometry & Angles (radians) & Weights\\
\midrule
regular Y & $0.20,2.2944,4.3888$ & $1/3,1/3,1/3$\\
weak Y & $0.25,2.15,4.55$ & $0.05,0.40,0.55$\\
clustered four & $0.10,0.75,1.40,4.20$ & $0.25,0.25,0.25,0.25$\\
weak four & $0.05,1.18,3.02,5.10$ & $0.05,0.25,0.30,0.40$\\
\bottomrule
\end{tabular}
\end{table}
The plain network has three hidden SiLU layers of width 128 and
\LearnedPlainParameters{}
parameters.  The residual network has four residual SiLU blocks of width 64
and \LearnedResidualParameters{} parameters.  Both are time-conditioned and trained with the same
denoising score-matching construction, optimizer, batch size 512, and
\LearnedSeedCount{} seeds indexed from zero.  A clean point first selects a branch according to its weight and a
radius uniformly from $[0,4]$.  For each training item,
\begin{equation}
 \sigma\sim\operatorname{Unif}\{0.1,0.2,0.4\},\qquad
 Z\sim\mathcal N(0,I_2),\qquad
 \widetilde X=X/\sigma+Z,
\end{equation}
and the per-coordinate denoising loss is
\begin{equation}
 \mathcal L(\theta)=\frac12\E
 \left\|f_\theta(\widetilde X,\sigma)+Z\right\|_2^2.
\end{equation}
Optimization uses AdamW with learning rate $2\times10^{-3}$, weight decay
$10^{-6}$, and no learning-rate schedule.  Dataset generation, model
initialization, and minibatch/noise draws use deterministic separately offset
random streams.  The saved final-loss diagnostic is an exponential moving
average over the last 50 updates with coefficient 0.9; it is not used as
evidence for geometry recovery.

The standard matrix has $N\in\{512,8192\}$,
$\sigma\in\{0.1,0.2,0.4\}$, and 1,000 updates.  There are
\LearnedStandardModels{} distinct
standard trained models, each evaluated at all three noise scales.  The long
matrix has $N=8192$, evaluation $\sigma=0.1$, 5,000 updates from the same
initial seed as its standard counterpart, and \LearnedLongModels{} trained
models.  Together with matched population and exact empirical KDE fields, this
yields \LearnedEvaluationRows{} saved rows and \LearnedOptimizerUpdates{}
optimizer updates.

The two complete learned-score scripts train \LearnedTotalModels{} models for
\LearnedTotalUpdates{} optimizer updates.  On the reported RTX 5070 Ti, the
training-budget sensitivity script takes \SensitivityWallMinutes{} minutes and
the cross-geometry script takes \LearnedEvaluationWallMinutes{} minutes
(\LearnedTotalWallMinutes{} minutes in total); the larger PyTorch peak CUDA
reservation is \LearnedPeakCudaReservedMiB{} MiB.  The CPU fallback is
exercised by the released pilots on the reported Ryzen 5 5600: the sensitivity
pilot takes \SensitivityCpuPilotSeconds{} seconds for
\SensitivityCpuPilotUpdates{} updates, while the cross-geometry pilot takes
\LearnedEvaluationCpuPilotSeconds{} seconds for
\LearnedEvaluationCpuPilotUpdates{} updates.
These are observed costs, not extrapolated full-CPU runtimes.  Run-metadata CSV
rows record the software versions, device, driver, workload, wall time, and both
peak CUDA allocation and reservation.

The seed-level outcomes are reported directly, without a combined study-level
decision rule.  At 5,000 updates, both parameterizations recover count in at
least \LearnedStagewiseCountFloor{} seeds but full geometry in at most
\LearnedStagewiseFullCeiling{} on regular Y, clustered four, and weak
four.  On weak Y, count is correct in \WeakYPlainLongCount{} plain MLP seeds
and \WeakYResidualLongCount{} residual MLP seeds.  From 1,000 to 5,000
updates, count improves for both parameterizations on clustered four
and weak four; regular Y is already count-correct at the shorter budget.

An oracle-assisted learned-field rank-screening budget is reported as a
deliberately conservative diagnostic.  It is
\begin{equation}
 \eta_{\rm rank}=\eta_{\rm learned\to empirical}
 +\eta_{\rm empirical\to population}
 +\eta_{\rm population\to truth},
\end{equation}
and it is compared with half the true synthetic Toeplitz signal gap.  Every
long learned field fails this inequality, so the diagnostic always abstains.
Because both the population-to-truth term and the gap use the known synthetic
geometry, this is not a deployable confidence statement and it does not execute
the root or weight tests.  It records only that this worst-case triangle budget
is too conservative to validate individual learned reconstructions; it does
not replace the directly evaluated count, angle, and weight outcomes.

Those outcomes are consistent with the decomposition in
Proposition~\ref{prop:stagewise}.
Clustered four reaches \ClusterPlainLongCount{} and
\ClusterResidualLongCount{} count-correct seeds in the plain and residual
parameterizations while its median angular or weight error crosses the declared
full-geometry tolerance.  Weak four reaches \WeakFourPlainLongCount{}
count-correct seeds in the plain network and \WeakFourResidualLongCount{} in
the residual network;
its median angular error is within the declared tolerance, but
its median weight error remains outside it.  Weak-Y count recovery varies
across seeds under this budget.  Across the tested geometries, correct count
recovery can coexist with angular or weight errors above tolerance.

\begin{table}[H]
  \centering
  \caption{Fixed learned-score evaluation (successes out of
  \LearnedSeedCount{} seeds).  ``Emp.'' is exact empirical KDE full recovery;
  count and full columns show learned fields at 1k$\to$5k updates.  Angle
  (degrees) and weight columns are medians over count-correct 5k seeds only.}
  \label{tab:learned}
  \resizebox{\linewidth}{!}{\begin{tabular}{llrccrr}
\toprule
Geometry & Network & Emp. & Count 1k$\to$5k & Full 1k$\to$5k & Angle 5k & Weight 5k \\
\midrule
clustered four & plain & 4/5 & 3/5$\to$5/5 & 0/5$\to$1/5 & 6.501 & 0.055 \\
clustered four & residual & 4/5 & 3/5$\to$5/5 & 0/5$\to$0/5 & 5.344 & 0.068 \\
regular Y & plain & 5/5 & 5/5$\to$5/5 & 2/5$\to$2/5 & 1.853 & 0.065 \\
regular Y & residual & 5/5 & 5/5$\to$5/5 & 2/5$\to$1/5 & 2.221 & 0.092 \\
weak four & plain & 5/5 & 3/5$\to$5/5 & 1/5$\to$0/5 & 2.505 & 0.099 \\
weak four & residual & 5/5 & 2/5$\to$4/5 & 0/5$\to$0/5 & 3.083 & 0.111 \\
weak Y & plain & 5/5 & 1/5$\to$3/5 & 1/5$\to$2/5 & 1.836 & 0.048 \\
weak Y & residual & 5/5 & 0/5$\to$2/5 & 0/5$\to$0/5 & 2.217 & 0.072 \\
\bottomrule
\end{tabular}
}
\end{table}

\begin{figure}[H]
  \centering
  \includegraphics[width=0.99\linewidth]{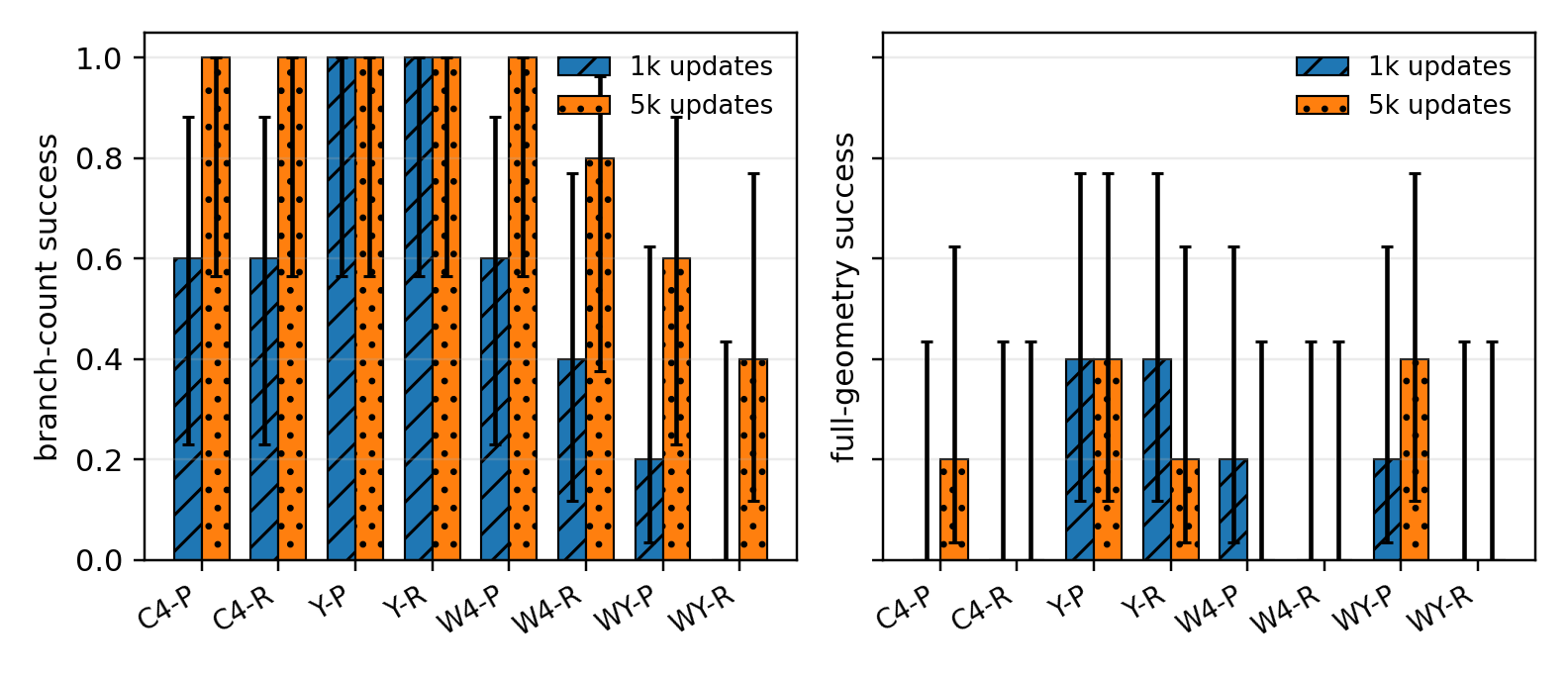}
  \caption{Fixed evaluation across four geometries and two parameter-matched
  multilayer-perceptron (MLP) parameterizations.  Left: count recovery; right:
  full-geometry recovery.
  Hatching distinguishes 1k from 5k updates, and error bars are 95\% Wilson
  intervals over \LearnedSeedCount{} seeds.  C4, Y, W4, and WY denote clustered
  four, regular Y, weak four, and weak Y; P and R denote plain and residual MLPs.}
  \label{fig:learned}
\end{figure}

\subsection{Train-to-convergence evaluation}

The 1k-to-5k comparison above asks whether additional optimization stabilizes
count recovery under the fixed inverse.  The long-horizon study instead asks
whether further improvements in normalized-score RMS also improve recovered
moments and full geometry.

The convergence study retains the four geometries in
Table~\ref{tab:learned-geometries}, fixes $N=8192$, and uses ten seeds per
architecture--geometry cell.  Plain and residual MLPs are parameter matched as
above.  A larger residual MLP with 132,866 parameters provides a capacity
control on regular Y, weak Y, and clustered four.  Every model trains
through 50,000 optimizer updates, with evaluation every 500 updates and saved
checkpoints at 1,000, 5,000, 10,000, 20,000, and 50,000 updates.  Each
trajectory uses one sampled training set.  Validation and test evaluations use
disjoint rotated shell grids against the analytic population score; empirical
score diagnostics reuse the sampled training set.

The plateau diagnostic is the first evaluation at which the best validation
normalized-score RMS improves by less than 1\% across ten consecutive
evaluations.  It is recorded only after 20,000 updates, and all models continue
to 50,000 updates.  The reported checkpoint minimizes mean validation
normalized-score RMS over the ten seeds and three trained noise scales among
the saved checkpoints at or beyond 20,000 updates.  Neither moment error nor a
geometry-recovery outcome participates in selection.

All \ConvergenceModels{} trajectories satisfy the plateau rule.  For the main
plain/residual comparison, validation selects the plain MLP at 50,000 updates
for every geometry.  From 5,000 updates to that checkpoint, normalized-score
RMS improves in every geometry, whereas maximum moment error increases in every
geometry.  The larger residual control also has higher validation score error
than the selected plain model on each of its three geometries.  Within this
protocol, the downstream errors were not resolved by training through 50,000
updates or by the tested roughly fourfold parameter increase; other schedules
and architectures remain untested.

\begin{table}[H]
  \centering
  \small
  \caption{Validation-selected train-to-convergence results at $\sigma=0.1$
  (ten seeds per row).  ``Score RMS'' is the root-mean-square error of the
  normalized score against its population value; moment error is the maximum
  error through order $K$.  Full recovery
  requires correct count, maximum angle error at most $5^\circ$, and maximum
  weight error at most $0.05$.}
  \label{tab:learned-convergence}
  \resizebox{\linewidth}{!}{\begin{tabular}{llrrrrr}
\toprule
Geometry & Arch. & Updates & Score RMS & Moment err. & Count & Full \\
\midrule
clustered four & plain & 50,000 & 0.0824 & 0.105 & 10/10 & 2/10 \\
regular Y & plain & 50,000 & 0.0993 & 0.133 & 10/10 & 2/10 \\
weak four & plain & 50,000 & 0.104 & 0.132 & 10/10 & 4/10 \\
weak Y & plain & 50,000 & 0.12 & 0.165 & 6/10 & 2/10 \\
\bottomrule
\end{tabular}
}
\end{table}

\begin{figure}[H]
  \centering
  \includegraphics[width=0.99\linewidth]{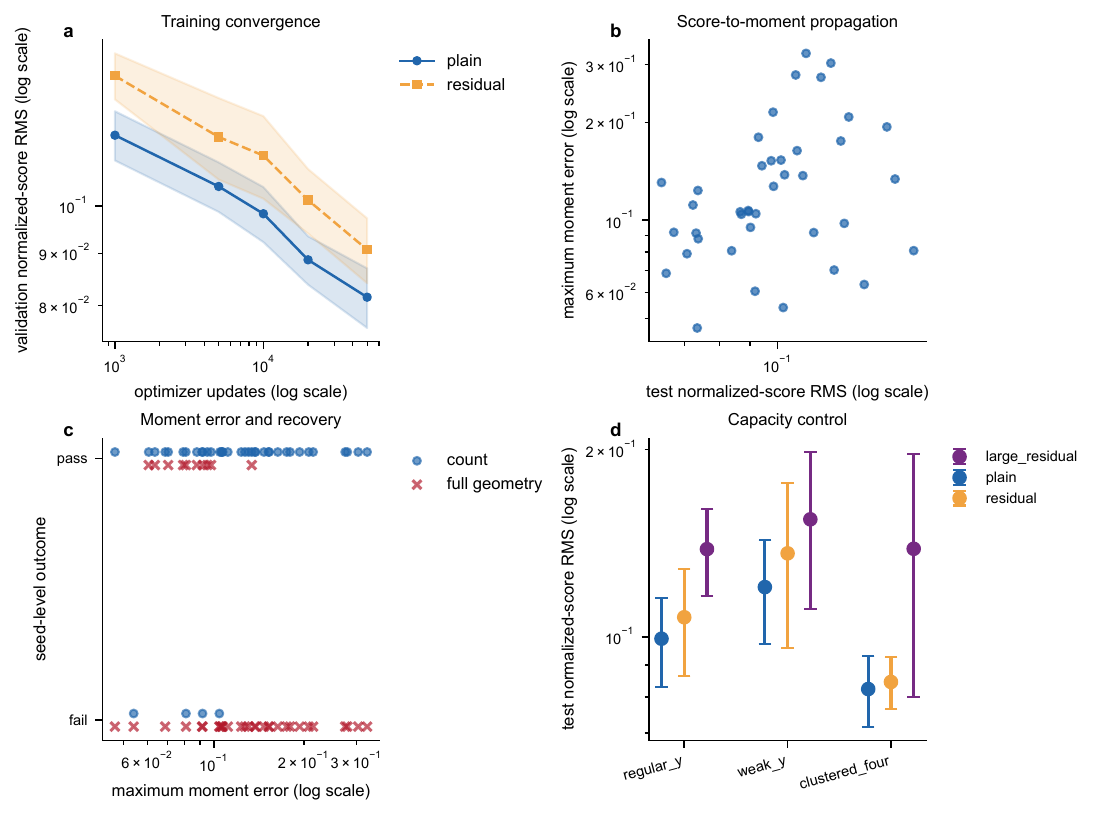}
  \caption{Train-to-convergence learned-score diagnostics.  (a) Validation
  normalized-score RMS across optimizer updates.  (b) Test normalized-score
  RMS versus maximum angular-moment error at saved checkpoints.  (c) Seed-level
  count and full-geometry outcomes versus moment error.  (d) Validation-selected
  comparison of parameter-matched plain/residual MLPs with a larger residual
  MLP on three geometries.  Curves and intervals aggregate ten seeds; checkpoint
  selection uses validation score error only.}
  \label{fig:learned-convergence}
\end{figure}

\subsection{Tolerance sensitivity from the released rows}

The main evaluation declares $5^\circ$ maximum angular error and $0.05$
maximum weight error.  To expose the dependence on that choice without
retraining or adding trials, we reaggregate the same \LearnedLongModels{}
long-run learned fields over a grid of angle and weight tolerances.

\begin{table}[H]
  \centering
  \caption{Pooled full-geometry successes over the released long-run learned
  fields.  Rows give the maximum angle error and columns the maximum weight
  error; each entry is successes out of \LearnedLongModels{} fields.  These
  fields span four geometries, two parameterizations, and
  \LearnedSeedCount{} seeds, so the pooled entries are descriptive
  reaggregations rather than additional independent replications.}
  \label{tab:threshold-sensitivity}
  \begin{tabular}{lccc}
\toprule
Angle tolerance & $0.02$ & $0.05$ & $0.10$ \\
\midrule
$2^\circ$ & 0/40 & 2/40 & 7/40 \\
$5^\circ$ & 1/40 & 6/40 & 19/40 \\
$10^\circ$ & 1/40 & 8/40 & 23/40 \\
\bottomrule
\end{tabular}

\end{table}

The changes across columns show that the absolute full-recovery fraction is
sensitive to the weight tolerance.  Failures remain even at the most
permissive displayed pair, while the stricter pairs expose different root and
weight bottlenecks.  We therefore retain the separate count, angle, and weight
measurements as the primary evidence rather than treating any one threshold as
a universal success definition.

\section{Scope of the results}
\label{app:scope}

Together, the proofs, crossing/gap derivations, and counterexamples establish
the following scope:
\begin{itemize}
\item \textbf{Proved:} weak single-scale recovery of center and homogeneity
under a full-rank test matrix; exact one-shell injectivity; constructive
at-most-$K$ finite-ray recovery in arbitrary dimension from moments through
degree $2K-1$; the $KD-1$ fixed scalar-query lower bound; sharp degree-$K$
planar recovery; finite-query moment, count, direction, and weight tests; and
the finite-data local-coverage bound under the stated assumptions.  For finite
planar $C^{1,\beta}$ junctions with $C^{0,\beta}$ positive densities, the
Gaussian-weighted tangent moments and normalized score converge at rate
$O(\sigma^\beta)$ on fixed compact normalized query sets.
\item \textbf{Observed synthetically:} weak calibration in $D=2,3,5$ under
exact, perturbed, and GPU-trained three-dimensional score fields;
non-coplanar three-dimensional score-shell recovery; held-out weighted-geometry
recovery beyond fixed templates; inverse-square-root coverage scaling; valid
execution of the complete bounded-error certificate; and the separation
between count and full-geometry fidelity across four planar geometries and two
MLP parameterizations.  The train-to-convergence extension further observes
that, across four tested geometries, the lower validation-selected
normalized-score RMS at 50k updates can coexist with higher angular-moment error
than at 5k updates.  The oracle rank gate on learned
fields is reported separately from these metrics.
\item \textbf{Ruled out under the stated score-oracle formulation:}
heterogeneous feature sizes alone implying multiscale necessity, universal
fixed-scale failure, irreversible branch loss, and far-field density decay
implying a weak log-score.  We also do not claim universality across diffusion
architectures or real data.
\end{itemize}

All reported numbers, tables, and figures are generated from executable code;
the supporting-material README lists the environment, commands, and output
files.  Seeded neural optimization may vary across hardware, so the released
raw rows can be reaggregated without retraining.

\end{document}